\documentclass[12pt]{article}
\usepackage{amsmath}
\usepackage{times}
\usepackage{graphicx}
\usepackage{color}
\usepackage{multirow}
\usepackage[authoryear]{natbib}
\usepackage{rotating}
\usepackage{bbm}
\usepackage{latexsym}
\usepackage{algorithm} 
\usepackage{algorithmic}
\usepackage{amsthm}
\newtheorem{theorem}{Theorem}
\newtheorem{proposition}{Proposition}
\newtheorem{definition}{Definition}

\newtheorem{lemma}{Lemma}
\newtheorem{remark}{Remark}
\usepackage{amssymb,amsfonts}

\newcommand{\captionfonts}{\normalsize}

\makeatletter  
\long\def\@makecaption#1#2{%
  \vskip\abovecaptionskip
  \sbox\@tempboxa{{\captionfonts #1: #2}}%
  \ifdim \wd\@tempboxa >\hsize
    {\captionfonts #1: #2\par}
  \else
    \hbox to\hsize{\hfil\box\@tempboxa\hfil}%
  \fi
  \vskip\belowcaptionskip}
\makeatother   

\begin{document}

\hspace{13.9cm}1

\ \vspace{20mm}\\

{\LARGE Improving Generalization via Attribute Selection on Out-of-the-box Data}

\ \\
{\bf \large Xiaofeng Xu$^{\displaystyle 1, \displaystyle 2}$, Ivor W. Tsang$^{\displaystyle 2}$ and Chuancai Liu$^{\displaystyle 1, \displaystyle 3}$}\\
{$^{\displaystyle 1}$School of Computer Science and Engineering, Nanjing University of Science and Technology.}\\
{$^{\displaystyle 2}$Centre for Artificial Intelligence, University of Technology Sydney.}\\
{$^{\displaystyle 3}$Collaborative Innovation Center of IoT Technology and Intelligent Systems, Minjiang University.}\\
%


\thispagestyle{empty}
\markboth{}{NC instructions}
\ \vspace{-0mm}\\
%
\begin{center} {\bf Abstract} \end{center}
Zero-shot learning (ZSL) aims to recognize unseen objects (test classes) given some other seen objects (training classes), by sharing information of attributes between different objects. Attributes are artificially annotated for objects and treated equally in recent ZSL tasks. However, some inferior attributes with poor predictability or poor discriminability may have negative impacts on the ZSL system performance. This paper first derives a generalization error bound for ZSL tasks. Our theoretical analysis verifies that selecting the subset of key attributes can improve the generalization performance of the original ZSL model, which utilizes all the attributes. Unfortunately, previous attribute selection methods are conducted based on the seen data, and their selected attributes have poor generalization capability to the unseen data, which is unavailable in the training stage of ZSL tasks. Inspired by learning from pseudo relevance feedback, this paper introduces the out-of-the-box data, which is pseudo data generated by an attribute-guided generative model, to mimic the unseen data. After that, we present an iterative attribute selection (IAS) strategy which iteratively selects key attributes based on the out-of-the-box data. Since the distribution of the generated out-of-the-box data is similar to the test data, the key attributes selected by IAS can be effectively generalized to test data. Extensive experiments demonstrate that IAS can significantly improve existing attribute-based ZSL methods and achieve state-of-the-art performance.

\section{Introduction}
With the rapid development of machine learning technologies, especially the rise of deep neural network, visual object recognition has made tremendous progress in recent years \citep{zheng2018sparse,shen2018weakly}. These recognition systems even outperform humans when provided with a massive amount of labeled data. However, it is expensive to collect sufficient labeled samples for all natural objects, especially for the new concepts and many more fine-grained subordinate categories \citep{zhou2019learning}. Therefore, how to achieve an acceptable recognition performance for objects with limited or even no training samples is a challenging but practical problem \citep{palatucci2009zero}. Inspired by human cognition system that can identify new objects when provided with a description in advance \citep{murphy2004big}, zero-shot learning (ZSL) has been proposed to recognize unseen objects with no training samples \citep{cheng2018random,ji2019attribute}. Since labeled sample is not given for the target classes, we need to collect some source classes with sufficient labeled samples and find the connection between the target classes and the source classes.

As a kind of semantic representation, attributes are widely used to transfer knowledge from the seen classes (source) to the unseen classes (target) \citep{ma2017joint}. Attributes play a key role in sharing information between classes and govern the performance of zero-shot classification. In previous ZSL works, all the attributes are assumed to be effective and treated equally. However, as pointed out in \citet{guo2018zero}, different attributes have different properties, such as the distributive entropy and the predictability. The attributes with poor predictability or poor discriminability may have negative impacts on the ZSL system performance. The poor predictability means that the attributes are hard to be correctly recognized from the feature space, and the poor discriminability means that the attributes are weak in distinguishing different objects. Hence, it is obvious that not all the attributes are necessary and effective for zero-shot classification.

Based on these observations, selecting the key attributes, instead of using all the attributes, is significant and necessary for constructing ZSL models. \citet{guo2018zero} proposed the zero-shot learning with attribute selection (ZSLAS) model, which selects attributes by measuring the distributive entropy and the predictability of attributes based on the training data. ZSLAS can improve the performance of attribute-based ZSL methods, while it suffers from the drawback of generalization. Since the training classes and the test classes are disjoint in ZSL tasks, the training data is bounded by the box cut by attributes (illustrated in Figure \ref{fig_box}). Therefore, the attributes selected based on the training data have poor generalization capability to the unseen test data.

\begin{figure}[t] 
	\hfill
	\begin{center}
		\includegraphics[width = 0.6\textwidth]{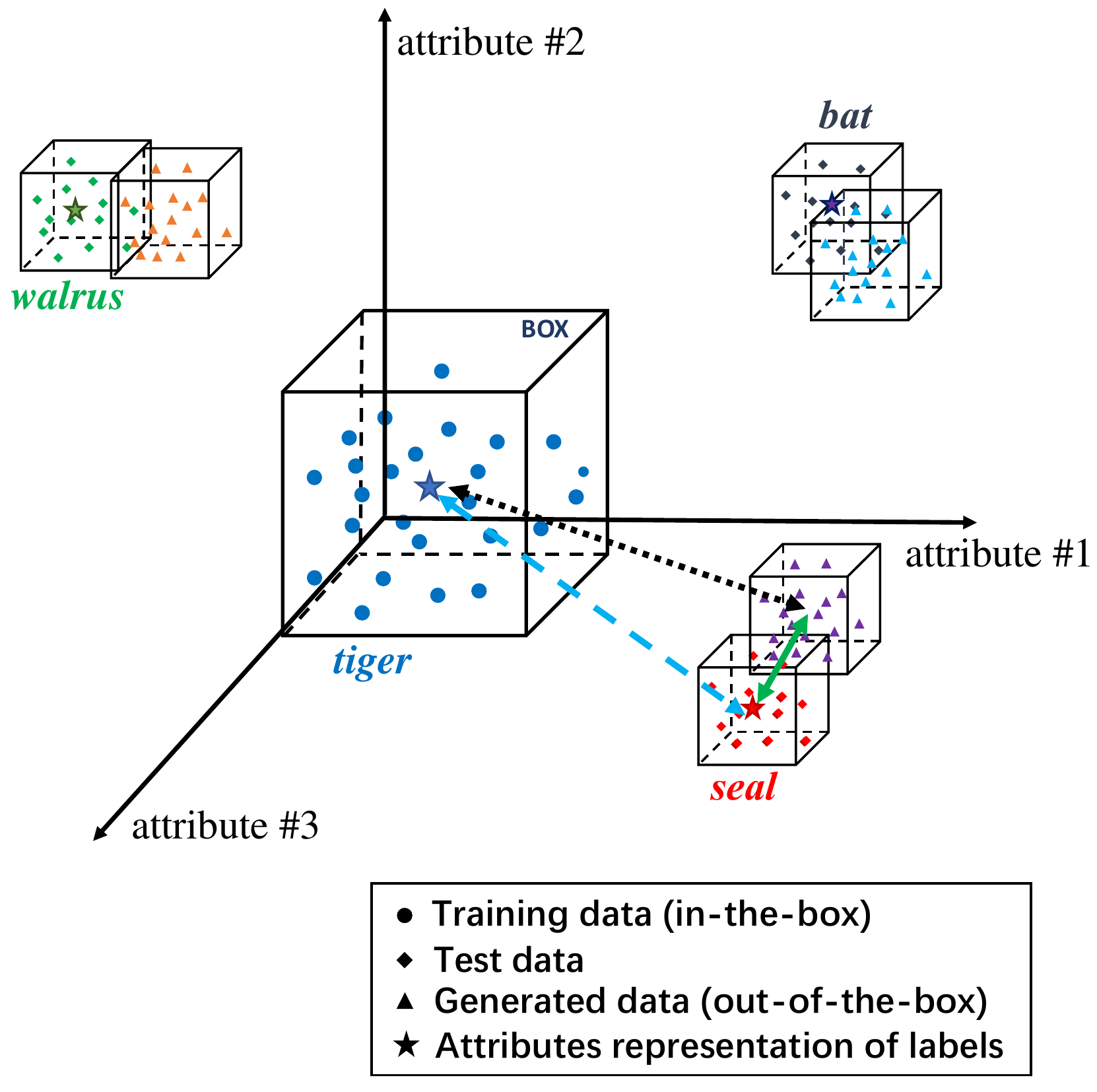}  
	\end{center}
	\caption{Illustration of out-of-the-box data. The distance between the out-of-the-box data and the test data (green solid arrow) is much less than the distance between the training data and the test data (blue dashed arrow).}
	\label{fig_box}
\end{figure}

To address the drawback, this paper derives a generalization error bound for ZSL problem. Since attributes for ZSL task is literally like the codewords in the error correcting output code (ECOC) model \citep{dietterich1994solving}, we analyze the bound from the perspective of ECOC. Our analyses reveal that the key attributes need to be selected based on the data which is out of the box (i.e. the distribution of the training classes). Considering that test data is unavailable during the training stage for ZSL tasks, inspired by learning from pseudo relevance feedback \citep{miao2016topprf}, we introduce the \textit{out-of-the-box}\footnote{The out-of-the-box data is generated based on the training data and the attribute representation without extra information, which follows the standard zero-shot learning setting.} data to mimic the unseen test classes. The out-of-the-box data is generated by an attribute-guided generative model using the same attribute representation as the test classes. Therefore, the out-of-the-box data has a similar distribution to the test data. 

Guided by the performance of ZSL model on the out-of-the-box data, we propose a novel \textit{iterative attribute selection} (IAS) model to select the key attributes in an iterative manner. 
Figure \ref{fig_zslias} illustrates the procedures of the proposed ZSL with iterative attribute selection (ZSLIAS).
Unlike the previous ZSLAS that uses training data to select attributes at once, our IAS first generates out-of-the-box data to mimic the unseen classes, and subsequently iteratively selects key attributes based on the generated out-of-the-box data. During the test stage, selected attributes are employed as a more efficient semantic representation to improve the original ZSL model. By adopting the proposed IAS, the improved attribute embedding space is more discriminative for the test data, and hence improves the performance of the original ZSL model.

\begin{figure*}[t] 
	\hfill
	\begin{center}
		\includegraphics[width = 0.9\textwidth]{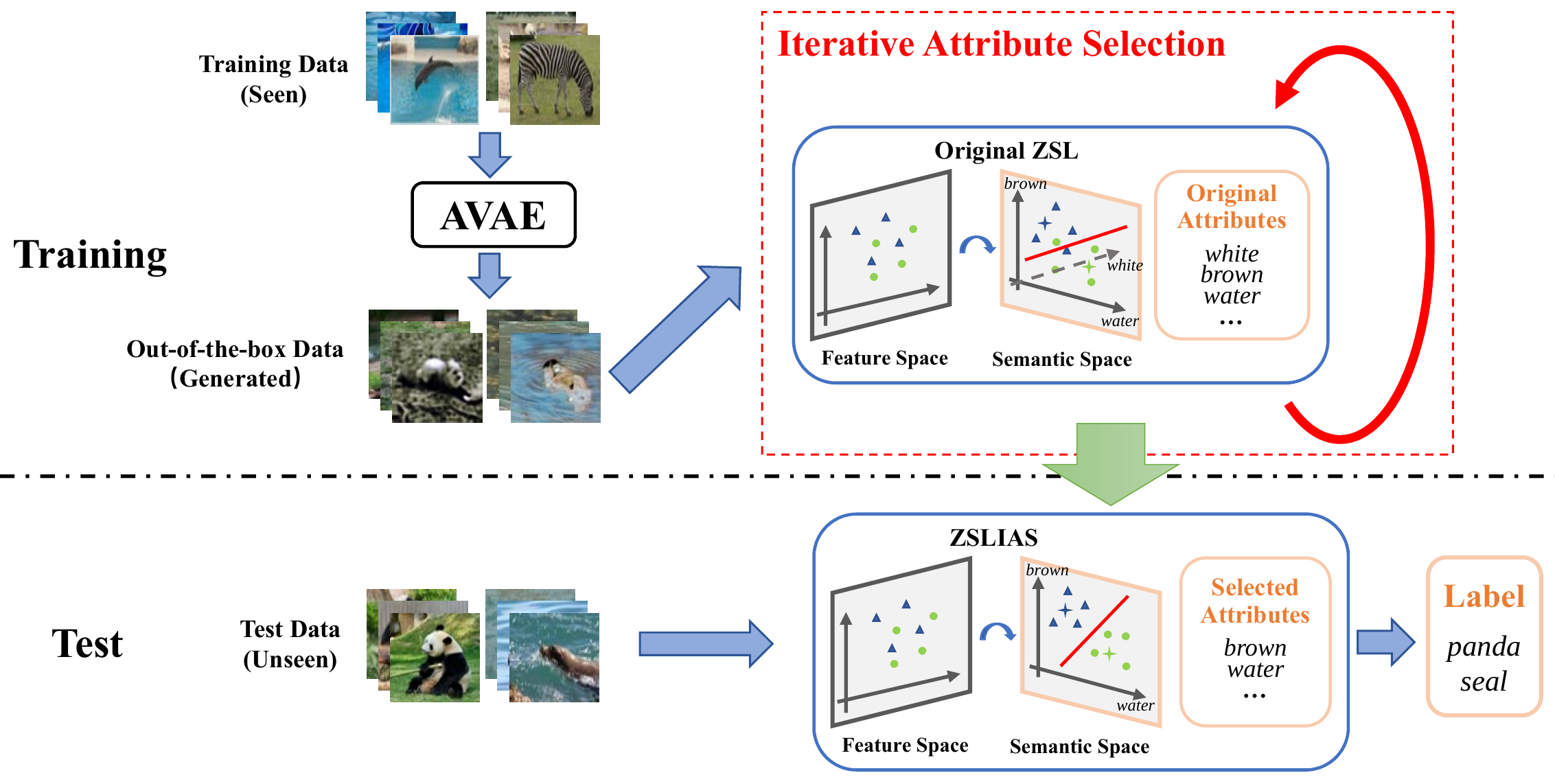}  
	\end{center}
	\caption{The pipeline of the ZSLIAS framework. \textit{In training stage}, we first generate the out-of-the-box data by a tailor-made generative model (i.e. AVAE), and then iteratively select attributes based on the out-of-the-box data. \textit{In test stage}, the selected attributes are exploited to build ZSL model for unseen objects categorization.}
	\label{fig_zslias}
\end{figure*}

The main contributions of this paper are summarized as follows: 

\begin{itemize}
	\item We present a generalization error analysis for ZSL problem. Our theoretical analyses prove that selecting the subset of key attributes can improve the generalization performance of the original ZSL model which utilizes all the attributes.
	\item Based on our theoretical findings, we propose a novel iterative attribute selection strategy to select key attributes for ZSL tasks.
	\item Since test data is unseen during the training stage for ZSL tasks, we introduce the out-of-the-box data to mimic test data for attribute selection. Such data generated by a designed generative model has a similar distribution to the test data. Therefore, attributes selected based on the out-of-the-box data can be effectively generalized to the unseen test data.
	\item Extensive experiments demonstrate that IAS can effectively improve the attribute-based ZSL model and achieve state-of-the-art performance. 
\end{itemize}

The rest of the paper is organized as follows. Section \ref{sec_relatedworks} reviews related works. Section \ref{sec_method} gives the preliminary and motivation. Section \ref{sec_bound} presents the theoretical analyses on generalization bound for attribute selection. Section \ref{sec_ias} proposes the iterative attribute selection model.  Experimental results are reported in Section \ref{sec_exper}. Conclusion is drawn in Section \ref{sec_conclu}.

\section{Related Works}
\label{sec_relatedworks}
In this section, we review some related works on zero-shot learning, attribute selection and deep generative models.

\subsection{Zero-shot Learning}
ZSL can recognize new objects using attributes as the intermediate semantic representation. Some researchers adopt the probability-prediction strategy to transfer information. \citet{lampert2014attribute} proposed a popular baseline, i.e. direct attribute prediction (DAP). DAP learns probabilistic attribute classifiers using the seen data and infers the label of the unseen data by combining the results of pre-trained classifiers. 

Most recent works adopt the label-embedding strategy that directly learns a mapping function from the input features space to the semantic embedding space. 
One line of works is to learn linear compatibility functions. For example, \citet{akata2016label} presented an attribute label embedding (ALE) model which learns a compatibility function combined with ranking loss. \citet{romera2015embarrassingly} proposed an approach that models the relationships among features, attributes and classes as a two linear layers network. Another direction is to learn nonlinear compatibility functions. \citet{xian2016latent} presented a nonlinear embedding model that augments bilinear compatibility model by incorporating latent variables. \citet{airola2018fast} proposed a first general Kronecker product kernel-based learning model for ZSL tasks. In addition to the classification task, \citet{ji2019attribute} proposed an attribute network for zero-shot hashing retrieval task.

\subsection{Attribute Selection}
Attributes, as a kind of popular semantic representation of visual objects, can be the appearance, a part or a property of objects \citep{farhadi2009describing}. For example, object \textit{elephant} has the attribute \textit{big} and \textit{long nose}, object \textit{zebra} has the attribute \textit{striped}. Attributes are widely used to transfer information to recognize new objects in ZSL tasks \citep{sun2017domain, xu2019complementary}. As shown in Figure \ref{fig_box}, using attributes as the semantic representation, data of different categories locates in different boxes bounded by the attributes. Since the attribute representation of the seen classes and the unseen class are different, the boxes with respect to the seen data and the unseen data are disjoint.

In previous ZSL works, all the attributes are assumed to be effective and treated equally. However, as pointed out in \citet{guo2018zero}, not all the attributes are effective for recognizing new objects. Therefore, we should select the key attributes to improve the semantic presentation. \citet{liu2014automatic} proposed a novel greedy algorithm which selects attributes based on their discriminating power and reliability. \citet{guo2018zero} proposed to select attributes by measuring the distributive entropy and the predictability of attributes based on the training data. In short, previous attribute selection models are conducted based on the training data, which makes the selected attributes have poor generalization capability to the unseen test data. While our IAS iteratively selects attributes based on the out-of-the-box data which has a similar distribution to the test data, and thus the key attributes selected by our model can be more effectively generalized to the unseen test data.

\subsection{Attribute-guided Generative Models}
Deep generative models \citep{ma2017pose} aim to estimate the joint distribution $ p (y; x) $ of samples and labels, by learning the class prior probability $ p (y) $ and the class-conditional density $ p (x | y) $ separately. The generative model can be extended to a conditional generative model if the generator is conditioned on some extra information, such as attributes in the proposed method. \citet{odena2016conditional} introduced a conditional version of generative adversarial nets, i.e. CGAN, which can be constructed by simply feeding the data label. CGAN is conditioned on both the generator and discriminator and can generate samples conditioned on class labels. Conditional Variational Autoencoder (CVAE) \citep{sohn2015learning}, as an extension of Variational Autoencoder, is a deep conditional generative model for structured output prediction using Gaussian latent variables. We modify CVAE with the attribute representation to generate out-of-the-box data for the attribute selection.

\section{Preliminary and Motivation}
\label{sec_method}
\subsection{ZSL Task Formulation}
We consider zero-shot learning as a task that recognizes unseen classes which have no labeled samples available. Given a training set $D_{s}=\left \{ \left ( x_{n},y_{n} \right ),n= 1,...,N_{s} \right \}$, the task of traditional ZSL is to learn a mapping $f:\mathcal{X}\rightarrow \mathcal{Y}$ from the image feature space to the label embedding space, by minimizing the following regularized empirical risk:
\begin{equation}
L\left ( y,f\left ( x;\mathbf{W} \right ) \right ) = \frac{1}{N_{s}}\sum_{n=1}^{N_{s}}l\left ( y_{n},f\left ( x_{n};\mathbf{W} \right ) \right )+\Omega \left ( \mathbf{W} \right ),
\label{eq_lossbasiczsl}
\end{equation}
where $l\left ( \cdot \right )$ is the loss function, which can be square loss $1/2(f(x)-y)^2$, logistic loss $\mathrm{log}(1+\mathrm{exp}(-yf(x)))$ or hinge loss $\mathrm{max}(0,1-yf(x))$. $\mathbf{W}$ is the parameter of mapping $f$, and $\Omega\left ( \cdot \right )$ is the regularization term.

The mapping function $ f $ is defined as follows:
\begin{equation}
f\left ( x;\mathbf{W} \right )=\mathrm{\arg} \mathop{\max}_{y\in \mathcal{Y} } F\left ( x,y;\mathbf{W} \right ),
\label{eq:flesori}
\end{equation}
where the function $F:\mathcal{X} \times \mathcal{Y}\rightarrow \mathcal{R}$ is the bilinear compatibility function to associate image features and label embeddings defined as follows:
\begin{equation}
F\left ( x,y;\mathbf{W} \right )=\theta \left ( x \right )^{T}\mathbf{W} \varphi \left ( y \right ),
\label{eq:Flesori}
\end{equation}
where  $\theta \left ( x \right )$ is the image features, $ \varphi \left ( y \right ) $ is the label embedding (i.e. attribute representation).

We summarize some frequently used notations in Table \ref{table_notations}.

\begin{table}[t]
	\centering
	\renewcommand\arraystretch{1}
	\caption{Notations and Descriptions.}
	\label{table_notations}
	\setlength{\tabcolsep}{2mm}{
		\begin{tabular}{|c|c||c|c|}
			\hline
			\textbf{Notation} &\textbf{Description} &\textbf{Notation} &\textbf{Description} \\ \hline\hline
			$D_s$ &training data (seen) &$N_s$ &\#training samples \\ \hline
			$D_u$ &test data (unseen) &$N_u$ &\#test samples \\ \hline
			$D_g$ & out-of-the-box data &$N_g$ &\#generated samples \\ \hline
			$\mathcal{X}$ &image features &$d$    &\#dimension of features \\ \hline
			$\mathcal{Y}_s$ &training classes (seen)    &$K$ &\#training classes \\ \hline
			$\mathcal{Y}_u$ &test classes (unseen)       &$L$    &\#test classes \\ \hline	
			$\mathbf{A}$ & attribute matrix &$\mathbf{a}_y$      &attribute vector of label $ y $    \\ \hline		
			${N_a}$      &\#all the attributes   &$\mathcal{A}$ &set of original attributes    \\ \hline
			$\mathbf{s}$ &selection vector    &$\mathcal{S}$ &subset of selected attributes    \\ \hline 	
		\end{tabular}
	}
\end{table}

\subsection{Interpretation of ZSL Task}\label{sec:interpretation}
In traditional ZSL models, all the attributes are assumed to be effective and treated equally. While in previous works, some researchers pointed out that not all the attributes are useful and significant for zero-shot classification \citep{jiang2017learning}. To the best of our knowledge, there is no theoretical analysis for the generalization performance of ZSL tasks, let alone selecting informative attributes for unseen classes. To fill in this gap, we first derive the generalization error bound for ZSL models. 

The intuition of our theoretical analysis is to simply treat the attributes as a kind of error correcting output codes, then the prediction of ZSL tasks can be deemed as the assignment of class labels with respective pre-defined ECOC, which is the closest to the predicted ECOC problem \citep{rocha2014multiclass}. Based on this novel interpretation, we derive a theoretical generalization error bound of ZSL model as shown in Section \ref{sec_bound}. From the generalization bound analyses, we find that the discriminating power of attributes governs the performance of the ZSL model.

\subsection{Deficiency of ZSLAS}
Some attribute selection works have been proposed in recent years. \citet{guo2018zero} proposed the ZSLAS model that selects attributes based on the distributive entropy and the predictability of attributes using training data. Simultaneously considering the ZSL model loss function and attribute properties in a joint optimization framework, they selected attributes by minimizing the following loss function:
\begin{equation}
\begin{split}
L( y,f( x;\mathbf{s},\mathbf{W} ) ) = & \frac{1}{N_{s}}\sum_{n=1}^{N_{s}}\{l_{\mathrm{ZSL}}\left ( y_{n},f\left ( x_{n};\mathbf{s},\mathbf{W} \right ) \right ) \\
& + \alpha l_p(\theta(x_n),\varphi(y_n);\mathbf{s}) - \beta l_v(\theta(x_n),\mu;\mathbf{s})\},
\end{split}
\label{eq_zslas}
\end{equation}
where $\mathbf{s}$ is the weight vector of the attributes which will be further used for attribute selection. $\theta(\cdot)$ is the attribute classifier, $\varphi(y_n)$ is the attribute representation, $\mu$ is an auxiliary parameter. $l_{\mathrm{ZSL}}$ is the model based loss function for ZSL, i.e. $l(\cdot)$ as defined in Eq. \eqref{eq_lossbasiczsl}. $l_{p}$ is the attribute prediction loss which can be defined based on specific ZSL models and $l_{v}$ is the loss of variance which measures the distributive entropy of attributes \citep{guo2018zero}. After getting the weight vector $\mathbf{s}$ by optimizing Eq. \eqref{eq_zslas}, attributes can be selected according to $\mathbf{s}$ and then be used to construct ZSL model.

From our theoretical analyses in Section~\ref{sec_bound}, ZSLAS can improve the original ZSL model to some extent \citep{guo2018zero}. However, ZSLAS suffers from a drawback that the attributes are selected based on the training data.
Since the training and test classes are disjoint in ZSL tasks, 
it is difficult to measure the quality and contribution of attributes regarding discriminating the unseen test classes. Thus, the selected attributes by ZSLAS have poor generalization capability to the test data due to the domain shift problem.

\subsection{Definition of Out-of-the-box}
Since previous attribute selection models are conducted based on the bounded in-the-box data, the selected attributes have poor generalization capability to the test data. However, the test data is unavailable during the training stage.
Inspired by learning from pseudo relevance feedback \citep{miao2016topprf},
we introduce the pseudo data, which is outside the box of the training data, to mimic test classes to guide the attribute selection. Considering that the training data is bounded in the box by attributes, we generate the \textit{out-of-the-box} data using an attribute-guided generative model. Since the out-of-the-box data is generated based on the same attribute representation as test classes, the box of the generated data will overlap with the box of the test data. And consequently, the key attributes selected by the proposed IAS model based on the out-of-the-box data can be effectively generalized to the unseen test data.

\section{Generalization Bound Analysis}
\label{sec_bound}
In this section, we first derive the generalization error bound of the original ZSL model and then analyze the bound changes after attribute selection. In previous works, some generalization error bounds have been presented for the ZSL task. \citet{romera2015embarrassingly} transformed ZSL problem to the domain adaptation problem and then analyzed the risk bounds for domain adaptation. \citet{stock2018comparative} considered ZSL problem as a specific setting of pairwise learning and analyzed the bound by the kernel ridge regression model. However, these bound analysis are not suitable for ZSL model due to their assumptions. In this work, we derive the generalization bound from the perspective of ECOC model, which is more similar to the ZSL problem.

\subsection{Generalization Error Bound of ZSL}
Zero-shot classification is an effective way to recognize new objects which have no training samples available. The basic framework of ZSL model is using attribute representation as the bridge to transfer knowledge from seen objects to unseen objects. To simplify the analysis, we consider ZSL as a multi-class classification problem. Therefore, ZSL task can be addressed via an ensemble method which combines many binary attribute classifiers. Specifically, we pre-trained a binary classifier for each attribute separately in the training stage. To classify a new sample, all the attribute classifiers are evaluated to obtain an attribute codeword (a vector in which each element represents the output of an attribute classifier). Then we compare the predicted codeword to the attribute representations of all the test classes to retrieve the label of the test sample.

To analyze the generalization error bound of ZSL, we first define some distances in the attribute space, and then present a proposition of the error correcting ability of attributes.

\begin{definition}[Generalized Attribute Distance]
	\label{def1}
	Given the attribute matrix $ \mathbf{A} $ for associating labels and attributes, let $ \mathbf{a}_i $, $ \mathbf{a}_j $ denote the attribute representation of label $ y_i $ and $ y_j $ in matrix $ \mathbf{A} $ with length $ N_a $, respectively. Then the generalized attribute distance between $ \mathbf{a}_i $ and $ \mathbf{a}_j $ can be defined as 
	\begin{equation}
	d(\mathbf{a}_i,\mathbf{a}_j) = \sum_{m=1}^{N_a}\Delta (\mathbf{a}_i^{(m)},\mathbf{a}_j^{(m)}),
	\label{eq_genedist}
	\end{equation}
	where $ N_a $ is the number of attributes, $ \mathbf{a}_i^{(m)} $ is the $ m^{th} $ element in the attribute representation $ \mathbf{a}_i $ of the label $ y_i $. $ \Delta (\mathbf{a}_i^{(m)},\mathbf{a}_j^{(m)}) $ is equal to $ 1 $ if $ \mathbf{a}_i^{(m)} \neq \mathbf{a}_j^{(m)} $, $ 0 $ otherwise.
\end{definition}

We further define the minimum distance between any two attribute representations in the attribute space.

\begin{definition}[Minimum Attribute Distance]
	\label{def2}
	The minimum attribute distance $ \tau $ of matrix $ \mathbf{A} $ is the minimum distance between any two attribute representations $\mathbf{a}_i$ and $\mathbf{a}_j$ as follows:
	\begin{equation}
	\tau = \mathop{\min}_{i\neq j} d(\mathbf{a}_i,\mathbf{a}_j), \qquad \forall \  1\leq i,j \leq N_a.
	\label{eq_minidist}
	\end{equation}
\end{definition}

Given the definition of distance in the attribute space, we can prove the following proposition.

\begin{proposition}[Error Correcting Ability \citep{zhou2019n}]
	\label{pro1}
	Given the label-attribute correlation matrix $ \mathbf{A} $ and a vector of predicted attribute representation $ f(x) $ for an unseen test sample $ x $ with known label $y$. If $ x $ is incorrectly classified, then the distance between the predicted attribute representation $ f(x) $ and the correct attribute representation $ \mathbf{a}_y $ is greater than half of the minimum attribute distance $ \tau $, i.e.
	\begin{equation}
	d(f(x),\mathbf{a}_y) \geq \frac{\tau}{2}.
	\label{eq_errcorr}
	\end{equation}
\end{proposition}

\begin{proof}
	Suppose that the predicted attribute representation for test sample $ x $ with correct attribute representation $ \mathbf{a}_y $ is $ f(x) $, and the sample $ x $ is incorrectly classified to the mismatched attribute representation $ \mathbf{a}_r $, where $r\in{\mathcal{Y}_u \setminus \{y\}}$. Then the distance between $ f(x) $ and $ \mathbf{a}_y $ is greater than the distance between $ f(x) $ and $ \mathbf{a}_r $, i.e.,
	\begin{equation}
	d(f(x),\mathbf{a}_y) \geq d(f(x),\mathbf{a}_r).
	\end{equation}
	
	Here, the distance between attribute representation can be expanded as the element-wise summation based on Eq. \eqref{eq_genedist} as follows:
	\begin{equation}
	\sum_{m=1}^{N_a}\Delta (f^{(m)}(x),\mathbf{a}_y^{(m)}) \geq \sum_{m=1}^{N_a}\Delta (f^{(m)}(x),\mathbf{a}_r^{(m)}).
	\label{eq_apaineq}
	\end{equation}
	
	Then, we have:
	\begin{equation}
	\begin{split}
	d(f(x),\mathbf{a}_y) & = \sum_{m=1}^{N_a}\Delta (f^{(m)}(x),\mathbf{a}_y^{(m)}) \\
	& = \frac{1}{2}\sum_{m=1}^{N_a}\Big\{\Delta (f^{(m)}(x),\mathbf{a}_y^{(m)}) + \Delta (f^{(m)}(x),\mathbf{a}_y^{(m)})\Big\} \\
	& \overset{(\romannumeral1)}{\geq}  \frac{1}{2}\sum_{m=1}^{N_a}\Big\{\Delta (f^{(m)}(x),\mathbf{a}_y^{(m)})  + \Delta (f^{(m)}(x),\mathbf{a}_r^{(m)})\Big\} \\
	& \overset{(\romannumeral2)}{\geq} \frac{1}{2}\sum_{m=1}^{N_a}{ \Delta (\mathbf{a}_y^{(m)},\mathbf{a}_r^{(m)})} \\
	& =  \frac{1}{2}d(\mathbf{a}_y,\mathbf{a}_r) \overset{(\romannumeral3)}{\geq}  \frac{\tau}{2},
	\end{split}
	\end{equation}
	where $(\romannumeral1)$ follows Eq. \eqref{eq_apaineq}, $(\romannumeral2)$ is based on the triangle inequality of distance metric \citep{zhou2019n} and $(\romannumeral3)$ follows Eq. \eqref{eq_minidist}.
\end{proof}

From Proposition \ref{pro1}, we can find that, the predicted attribute representation is not required to be exactly the same as the ground truth for each unseen test sample. As long as the distance is less than $ {\tau}/{2} $, ZSL models can correct the error committed by some attribute classifiers and make an accurate prediction.

Based on the Proposition of error correcting ability of attributes, we can derive the theorem of generalization error bound for ZSL.

\begin{theorem}[Generalization Error Bound of ZSL]
	\label{the1}
	Given $ N_a $ attribute classifiers, $ f^{(1)}, $ $ f^{(2)}, ...,f^{(N_a)} $, trained on training set $ D_s $ with label-attribute matrix $ \mathbf{A} $, the generalization error rate for the attribute-based ZSL model is upper bounded by
	\begin{equation}
	\frac{2N_a\bar{B}}{\tau},
	\end{equation}
	where $ \bar{B}=\frac{1}{N_a}\sum_{m=1}^{N_a}B_m $ and $ B_m $ is the upper bound of the prediction loss for the $ m^{th} $ attribute classifier $ f^{(m)} $.
\end{theorem}

\begin{proof}
	According to Proposition \ref{pro1}, for any incorrectly classified test sample $ x $ with label $y$, the distance between the predicted attribute representation $ f(x) $ and the true attribute representation $ \mathbf{a}_y $ is greater than $ {\tau}/{2} $, i.e.,
	\begin{equation}
	d(f(x),\mathbf{a}_y) = \sum_{m=1}^{N_a}\Delta (f^{(m)}(x),\mathbf{a}_y^{(m)}) \geq \frac{\tau}{2}.
	\end{equation}
	
	Let $ k $ be the number of incorrect image classifications for unseen test dataset $ D_{u}= \{( x_{i},y_{i}  ),i= 1,...,N_{u} \} $, we can obtain:
	\begin{equation}
	\begin{split}
	k\frac{\tau}{2} &\leq \sum_{i=1}^{N_u}\sum_{m=1}^{N_a}\Delta(f^{(m)}(x_i),\mathbf{a}_{y_i}^{(m)}) \\
	& \leq \sum_{i=1}^{N_u}\sum_{m=1}^{N_a}B_m = N_u N_a \bar{B},
	\end{split} 
	\end{equation}
	where $ \bar{B}=\frac{1}{N_a}\sum_{m=1}^{N_a}B_m $ and $ B_m $ is the upper bound of attribute prediction loss.
	
	Hence, the generalized error rate ${k}/{N_u}$ is bounded by $ {2N_a\bar{B}}/{\tau}$.
\end{proof}

\begin{remark}[Generalization error bound is positively correlated to the average attribute prediction loss]
	\label{remark_t1}
	\emph{From Theorem \ref{the1}, we can find that the generalization error bound of the attribute-based ZSL model depends on the number of attributes $ N_a $, minimum attribute distance $ \tau $ and average prediction loss $ \bar{B} $ for all the attribute classifiers. According to the Definition \ref{def1} and \ref{def2}, the minimum attribute distance $ \tau $ is positively correlated to the number of attributes $ N_a $. Therefore, the generalization error bound is mainly affected by the average prediction loss $ \bar{B} $. Intuitively, the inferior attributes with poor predictability cause greater prediction loss $ \bar{B} $, and consequently, these attributes will have negative effect on the ZSL performance and increase the generalization error rate.}
\end{remark}

\subsection{Improvement of Generalization after Attribute Selection}
It has been proven that the generalization error bound of ZSL model is affected by the average prediction loss $ \bar{B} $ in the previous section. In this section, we will prove that attribute selection can reduce the average prediction loss $ \bar{B} $, and consequently reduce the generalization error bound of ZSL from the perspective of PAC-style \citep{valiant1984theory} analysis.
\begin{lemma}[PAC bound of ZSL \citep{palatucci2009zero}]
	\label{lemma1}
	Given $ N_a $ attribute classifiers, to obtain an attribute classifier with $ (1-\delta) $ probability that has at most $ k_a $ incorrect predicted attributes, the PAC bound $ D $ of the attribute-based ZSL model is:
	\begin{equation}
	D \propto  \frac{N_a}{k_a}[4 \mathrm{log}(2/\delta)+8(d+1)\mathrm{log}(13N_a/k_a)],
	\end{equation}
	where $ d $ is the dimension of the image features.
\end{lemma}

\begin{remark}[The average attribute prediction loss is positively correlated to the PAC bound]
	\label{remark_l1}
	\emph{Here, $ k_a/N_a $ is the tolerable prediction error rate of attribute classifiers. According to the definition of the average attribute prediction loss $ \bar{B} $, it is obvious that the ZSL model with smaller $ \bar{B}$ could tolerate a greater $ k_a/N_a $. From Lemma \ref{lemma1}, we can find that the PAC bound $ D $ is monotonically increasing with respect to $ N_a/k_a $. Hence, the PAC bound $ D $ decreases when the $ N_a/k_a $ decreases, and consequently the average prediction loss $ \bar{B} $ decreases.}
\end{remark}

\begin{lemma}[Test Error Bound \citep{vapnik2013nature}]
	\label{lemma2}
	Suppose that the PAC bound of the attribute-based ZSL model is $ D $. The probability of the test error distancing from an upper bound is given by:
	\begin{equation}
	p\Bigg(e_{ts} \leq e_{tr} +\sqrt{\frac{1}{N_s}\bigg[D\Big(\mathrm{log}\Big(\frac{2N_s}{D}\Big)+1\Big)-\mathrm{log}\Big(\frac{\eta }{4}\Big)\bigg]}\Bigg) = 1-\eta,
	\end{equation}
	where $ N_s $ is the size of the training set, $0 \leq \eta \leq 1$, and $ e_{ts} $,  $ e_{tr} $ are the test error and the training error respectively.
\end{lemma}

\begin{remark}[PAC bound is positively correlated to the test error bound]
	\label{remark_l2}
	\emph{From Lemma \ref{lemma2}, we can find that the PAC bound can affect the probabilistic upper bound on the test error. Specifically, to obtain a high probability with small test error, the PAC bound should be small. In other words, the model with smaller PAC bound would have a smaller test error bound.}
\end{remark}

\begin{proposition}[Bound Change after Attribute Selection]
	\label{theo2}
	For the attribute-based ZSL model, attribute selection can decrease the generalization error bound.
\end{proposition}

\begin{proof}
	In attribute selection, the key attributes are selected by minimizing the loss function in Eq. \eqref{eq_lossbasiczsl} on the out-of-the-box data. Since the generated out-of-the-box data has a similar distribution to the test data, the test error of ZSL will decrease after attribute selection, i.e. ZSLIAS has a smaller test error bound than the original ZSL model. Therefore, we can infer that ZSLIAS has a smaller PAC bound based on Remark \ref{remark_l2}. According to Remark \ref{remark_l1}, we can infer that the average prediction error $ \bar{B}$ decreases after attribute selection. As a consequence, the generalization error bound of ZSLIAS is smaller than the original ZSL model based on Remark \ref{remark_t1}.
\end{proof}

From Proposition \ref{theo2}, we can observe that the generalization error of ZSL model will decrease after adopting the proposed IAS. In other words, ZSLIAS have a smaller classification error rate comparing to the original ZSL method when generalizing to the unseen test data.

\section{IAS with Out-of-the-box Data}
\label{sec_ias}
Motivated by the generalization bound analyses, we select the key attributes based on the out-of-the-box data. In this section, we first present the proposed iterative attribute selection model. Then, we introduce the attribute-guided generative model designed to generate the out-of-the-box data. The complexity analysis of IAS is given at last.

\subsection{Iterative Attribute Selection Model}
Inspired by the idea of iterative machine teaching \citep{liu2017iterative}, we propose a novel iterative attribute selection model that iteratively selects attributes based on the generated out-of-the-box data. Firstly, we generate the out-of-the-box data to mimic test classes by an attribute-based generative model. Then, the key attributes are selected in an iterative manner based on the out-of-the-box data. After obtaining the selected attributes, we can consider them as a more efficient semantic representation to improve the original ZSL model. 

Suppose given the generated out-of-the-box data $D_{g} = \{( x_{n},y_{n} ), n= 1,...,N_{g}\}$, we can combine the empirical risk in Eq. \eqref{eq_lossbasiczsl} with the attribute selection model. Then the loss function is  rewritten as follows:
\begin{equation}
L\left ( y,f\left ( x;\mathbf{s},\mathbf{W} \right ) \right )  = \frac{1}{N_{g}}\sum_{n=1}^{N_{g}}l\left ( y_{n},f\left ( x_{n};\mathbf{s},\mathbf{W} \right ) \right )+\Omega \left ( \mathbf{W} \right ),
\label{eq:Lg}
\end{equation}
where $\mathbf{s}\in\left ( 0,1 \right )^{N_{a}}$ is the indicator vector for the attribute selection, in which $s_{i}=1$ if the $ i^{th} $ attribute is selected or $0$ otherwise. ${N_{a}}$ is the number of all the attributes.

Correspondingly, the mapping function $f$ in Eq. \eqref{eq:flesori} and the compatibility function $F$ in Eq. \eqref{eq:Flesori} can be rewritten as follows:
\begin{equation}
f\left ( x;\mathbf{s},\mathbf{W} \right )=\mathrm{\arg} \mathop{\max}_{y\in \mathcal{Y} } F\left ( x,y;\mathbf{s},\mathbf{W} \right ),
\end{equation}
\begin{equation}
F\left ( x,y;\mathbf{s},\mathbf{W} \right )=\theta \left ( x \right )^{T}\mathbf{W}\left ( \mathbf{s}\circ \varphi \left ( y \right ) \right ),
\label{eq:Fless}
\end{equation}
where $\circ$ is element-wise product operator (Hadamard product), $\mathbf{s}$ is the selection vector defined in Eq. \eqref{eq:Lg}.

To solve the optimization problem in Eq. \eqref{eq:Lg}, we need to specify the choice of the loss function $l\left ( \cdot \right )$. The loss function in Eq. \eqref{eq:Lg} for single sample $(x_n,y_n)$ is expressed as follows \citep{xian2018zero}:
\begin{align}
&l ( y_{n},f (  ( x_{n};\mathbf{s},\mathbf{W}  )  )  )\notag \\
\label{eq:LofLB} &= \sum_{y\in \mathcal{Y}_{g}}r_{ny} [ \triangle ( y_{n},y  )+F ( x_{n},y;\mathbf{s},\mathbf{W}  ) -F ( x_{n},y_{n};\mathbf{s},\mathbf{W}  ) ]_{+}  \\ 
&= \sum_{y\in \mathcal{Y}_{g}}r_{ny} [ \triangle( y_{n},y )+\theta ( x_{n} )^{T}\mathbf{W}( \mathbf{s}\circ \varphi (y) )-\theta ( x_{n} )^{T}\mathbf{W}( \mathbf{s}\circ \varphi ( y_{n} ) )   ]_{+}, \notag
\end{align}
where $ \mathcal{Y}_{g} $ is the label of generated out-of-the-box data, which is the same as  $ \mathcal{Y}_{u} $. \\ $\triangle(y_n;y)=0$ if $y_n = y$, 1 otherwise. $r_{ny}\in [0,1]$ is the weight defined in specific ZSL methods. 

Since the dimension of the optimal attribute subset (i.e. $l_0$-norm of $\mathbf{s}$) is agnostic, finding the optimal $\mathbf{s}$ is a NP-Complete \citep{garey1974some} problem. Therefore, inspired by the idea of iterative machine teaching \citep{liu2017iterative}, we adopt the greedy algorithm \citep{cormen2009introduction} to optimize the loss function in an iterative manner. Eq. \eqref{eq:Lg} gets updated during each iteration as follows:
\begin{equation}
\begin{split}
L^{t+1}=\frac{1}{N_{g}}\sum_{n=1}^{N_{g}} & l^{t+1}( y_{n},f ( x_{n};\mathbf{s}^{t+1},\mathbf{W}^{t+1}  )  )+\Omega  ( \mathbf{W}^{t+1}  ), \\
& s.t. \sum_{s_{i}\in \mathbf{s}^{t+1}}s_{i}=t+1 ,\\
& \;\;\sum_{s_{j}\in \left ( \mathbf{s}^{t+1}-\mathbf{s}^{t} \right )}s_{j}=1.
\label{eq:Lgt1}
\end{split}
\end{equation}
The constraints on $\mathbf{s}$ ensure that $\mathbf{s}^{t}$ updates one element (from 0 updates to 1) during each iteration, which indicates that only one attribute is selected each time. $\mathbf{s}^{0}$ is the initial vector of all 0's.

Correspondingly, the loss function in Eq. \eqref{eq:Lgt1} for single sample $(x_n,y_n)$ gets updated during each iteration as follows: 
\begin{equation}
\begin{split}
l^{t+1} = \sum_{y\in \mathcal{Y}_{g}} r_{ny} [ & \triangle ( y_{n},y  )+\theta  ( x_{n}  )^{T}\mathbf{W}^{t+1} ( \mathbf{s}^{t+1} \circ \varphi  ( y  )  )  \\
&\;\; -\theta  ( x_{n}  )^{T}\mathbf{W}^{t+1} ( \mathbf{s}^{t+1}\circ \varphi  ( y_{n}  )  )   ]_{+}.
\end{split}	
\label{eq:LofLBt1}
\end{equation}
Here $ l^{t+1} $ subjects to the same constrains as Eq. \eqref{eq:Lgt1}.

To minimize the loss function in Eq. \eqref{eq:Lgt1}, we can alternatively optimize $\mathbf{W}^{t+1}$ and $\mathbf{s}^{t+1}$ by optimizing one variable while fixing the other one. In each iteration, we firstly optimize $\mathbf{W}^{t+1}$ via the gradient descent algorithm \citep{burges2005learning}. The gradient of Eq. \eqref{eq:Lgt1} is calculated as follows:
\begin{equation}
\frac{\partial L^{t+1}}{\partial \mathbf{W}^{t+1}} = \frac{1}{N_{g}}\sum\limits_{n=1}^{N_{g}} \frac{\partial l^{t+1}}{\partial \mathbf{W}^{t+1}} + \frac{1}{2}\alpha\mathbf{W}^{t+1},
\label{eq:LdWt}
\end{equation}
where
\begin{equation}
\frac{\partial l^{t+1}}{\partial \mathbf{W}^{t+1}}=\sum_{y\in \mathcal{Y}_{g}}r_{ny}\theta (x_{n})^{T}(\mathbf{s}^{t}\circ (\varphi (y)-\varphi (y_n))),
\label{eq:ldWt}
\end{equation}
where $ \alpha $ is the regularization parameter.

After updating $\mathbf{W}^{t+1}$, we can traverse all the elements equal to $0$ in $\mathbf{s}^{t}$, and turn them into 1 respectively. Then $\mathbf{s}^{t+1}$ is updated by the optimal $\mathbf{s}^{t+1}$ which achieves the minimal loss of Eq. \eqref{eq:Lgt1}:
\begin{equation}
\label{updates}
\mathbf{s}^{t+1} = \mathrm{\arg}\mathop{\min}_{\mathbf{s}^{t+1}}\frac{1}{N_{g}}\sum_{n=1}^{N_{g}}  l^{t+1}( y_{n},f ( x_{n};\mathbf{s}^{t+1},\mathbf{W}^{t+1}  )  )+\Omega  ( \mathbf{W}^{t+1}  ),
\end{equation}

When iterations end and $\mathbf{s}$ is obtained, we can easily get the subset of key attributes by selecting the attributes corresponding to the elements equal to 1 in the selection vector $\mathbf{s}$. 

The procedure of the proposed IAS model is given in Algorithm \ref{alg:zslias}.

\begin{algorithm}[!t]
	\caption{Iterative Attribute Selection Model}
	\begin{algorithmic}[1]
		\REQUIRE ~~\\
		The generated out-of-the-box data $D_g$; \\
		Original attribute set $\mathcal{A}$; \\
		Iteration stop threshold $ \varepsilon $. \\
		\ENSURE ~~\\
		Subset of selected attributes $\mathcal{S}$. \\		
		\vspace{2mm}
		\STATE Initialization: $\mathbf{s}^0 = \mathbf{0}$, randomize $ \mathbf{W}^0 $;
		\STATE \textbf{for} $ t=0 $ to $ N_a - 1 $ \textbf{do}
		\STATE \quad \quad  $ L^t = \frac{1}{N_{g}}\sum\limits_{n=1}^{N_{g}} l^{t}( y_{n},f ( x_{n};\mathbf{s}^{t},\mathbf{W}^{t}  )  )+\Omega  ( \mathbf{W}^{t}  )$ (Eq. \eqref{eq:Lgt1})
		\STATE \quad \quad  $ \frac{\partial L^{t}}{\partial \mathbf{W}^{t}} = \frac{1}{N_{g}}\sum\limits_{n=1}^{N_{g}} \frac{\partial l^{t}}{\partial \mathbf{W}^{t}}  + \frac{1}{2}\alpha\mathbf{W}^t$ (Eq. \eqref{eq:LdWt})
		\STATE \quad  // {Update} $ \mathbf{W} $
		\STATE \quad \quad $ \mathbf{W}^{t+1} = \mathbf{W}^{t} - \eta_t \frac{\partial L^{t}}{\partial \mathbf{W}^{t}}$ 
		\STATE \quad  // {Update} $ \mathbf{s} $ 
		\STATE \quad \quad  $ \mathbf{s}^{t+1} = \mathrm{\arg}\mathop{\min}\limits_{\mathbf{s}^{t+1}}L^{t+1} $ (Eq. \eqref{updates})
		
		\STATE \quad \quad \textbf{if} $ |L^{t+1} - L^{t}| \leq \varepsilon  $
		\STATE \quad \quad \quad \quad \textbf{Break;}
		\STATE \quad \quad \textbf{end if}
		\STATE \textbf{end for}
		\STATE Obtain the subset of selected attributes:  $\mathcal{S}=\mathbf{s} \circ \mathcal{A}$.
	\end{algorithmic} 
	\label{alg:zslias}	
\end{algorithm}

\subsection{Generation of Out-of-the-box Data}
In order to select the discriminative attributes for test classes, we should do attribute selection on the test data. Since the training data and the test data are located in the different boxes bounded by the attributes, we adopt an attribute-based generative model \citep{bucher2017generating} to generate out-of-the-box data to mimic test classes. Comparing to the ZSLAS, the key attributes selected by IAS based on the out-of-the-box data can be more efficiently generalized to test data.

Conditional variational autoencoder (CVAE)~\citep{sohn2015learning} is a conditional generative model in which the latent codes and generated data are both conditioned on some extra information. In this work, we propose the attribute-based variational autoencoder (AVAE), a special version of CVAE with tailor-made attributes, to generate the out-of-the-box data.

VAE \citep{kingma2013auto} is a directed graphical model with certain types of latent variables. The generative process of VAE is as follows: a set of latent codes $z$ is generated from the prior distribution $p(z)$, and the data $x$ is generated by the generative distribution $p(x|z)$ conditioned on $z: z \sim  p(z)$, $x \sim p(x|z)$. The empirical objective of VAE is expressed as follows \citep{sohn2015learning}:
\begin{equation}
L_{\mathrm{VAE}}(x )=-\mathrm{KL}(q(z|x)\parallel p(z))+\frac{1}{L}\sum_{l=1}^{L}\mathrm{log} p(x|z^{(l)}),
\label{eq:vae}
\end{equation}
where $z^{(l)}=g(x,\epsilon ^{(l)})$, $\epsilon ^{(l)}\sim \mathcal{N}(\mathbf{0},\mathbf{I})$. $q(z|x)$ is the recognition distribution which is reparameterized with a deterministic and differentiable function $ g(\cdot,\cdot) $ \citep{sohn2015learning} . $ \mathrm{KL} $ denotes the Kullback-Leibler divergence \citep{kullback1987letter} between the incorporated distributions. $ L $ is the number of samples.

Combining with the condition, i.e. the attribute representation of labels, the empirical objective of the AVAE is defined as follows:
\begin{equation}
\begin{split}
L_{\mathrm{AVAE}}(x,\varphi (y))=-\mathrm{KL}(q(z|x,\varphi (y))\parallel p(z|\varphi (y))) \;+\frac{1}{L}\sum_{l=1}^{L}\mathrm{log} p(x|\varphi (y),z^{(l)}),
\end{split}
\label{eq:cvae}
\end{equation}
where $z^{(l)}=g(x,\varphi (y),\epsilon ^{(l)})$, $ \varphi \left ( y \right ) $ is the attribute representation of label $ y $.

In the encoding stage, for each training data point $x^{(i)}$, we estimate the $ q(z^{(i)}|x^{(i)}, \\ \varphi (y^{(i)})) = Q(z) $ using the encoder. In the decoding stage, after inputting the concatenation of the $\tilde{z}$ sampled from the $Q(z)$ and the attribute representation $\varphi (y_{u})$, the decoder will generate a new sample $x_{g} $ with the same attribute representation as the unseen class $\varphi (y_{u})$.

The procedure of AVAE is illustrated in Figure \ref{fig_cvae}. At training time, the attribute representation (of training classes) whose image is being fed in is provided to the encoder and decoder. To generate an image of a particular attribute representation (of test classes), we can just feed this attribute vector along with a random point in the latent space sampled from a standard normal distribution. The system no longer relies on the latent space to encode what object you are dealing with. Instead, the latent space encodes attribute information. Since the attribute representations of test classes are fed into the decoder at generating stage, the generated out-of-the-box data $D_g$ has a similar distribution to the test data.

\begin{figure}[t] 
	\hfill
	\begin{center}
		\includegraphics[width = 0.9\textwidth]{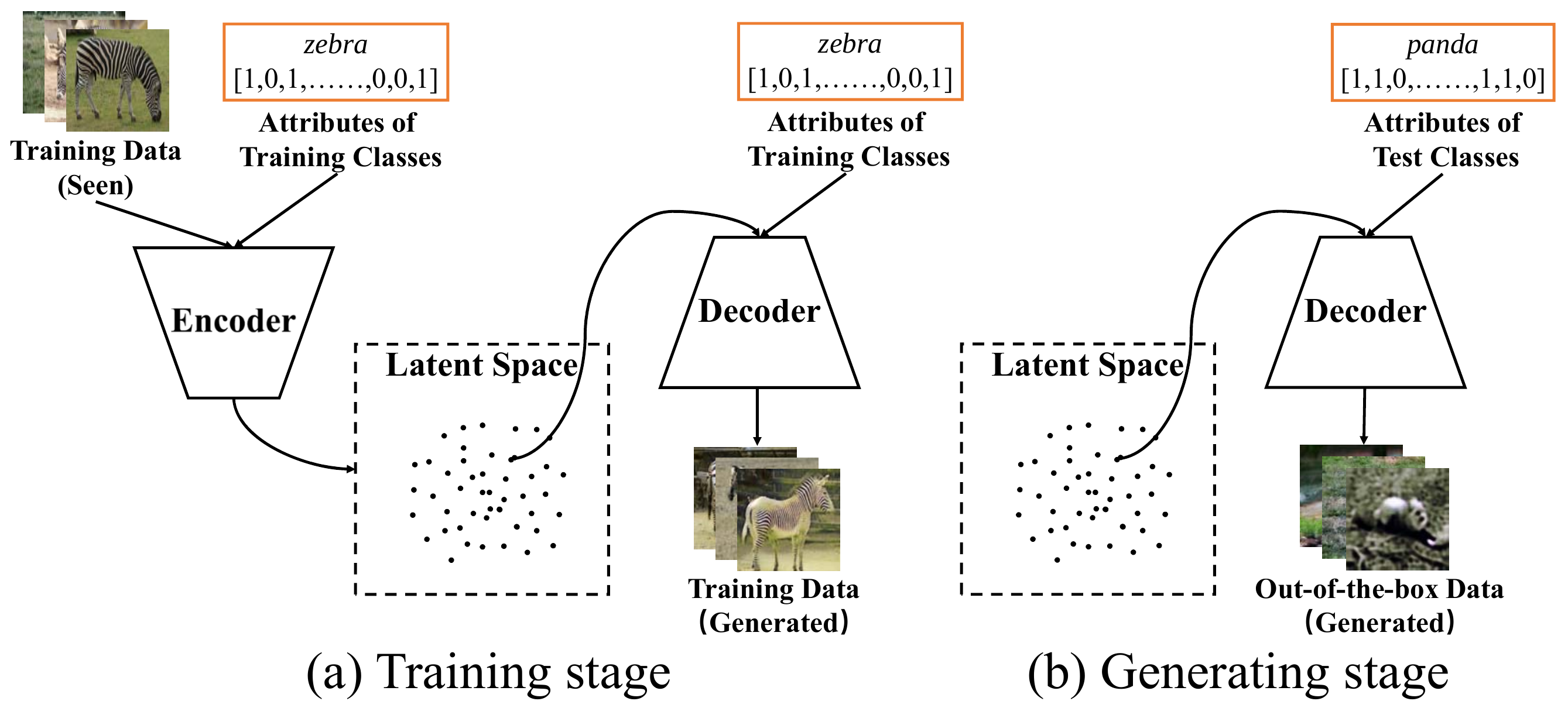}  
	\end{center}
	\caption{The framework of AVAE. (a) Training stage, (b) Generating stage.}
	\label{fig_cvae}
\end{figure}

\subsection{Complexity Analysis}
Suppose that there are $N_u$ unseen samples belonging to $L$ test classes, and the number of all the attributes is $N_a$. The complexity of original ZSL model is $\mathcal{O}_{\mathrm{ZSL}} \sim  \mathcal{O}(N_uN_aL^2)$. For the proposed ZSLIAS, the complexity of training stage is $\mathcal{O}_{\mathrm{ZSLIAS}} \sim N_a(N_a+1)/2\cdot\mathcal{O}_{\mathrm{ZSL}} $, i.e. $ \mathcal{O}(N_uN_a^3L^2)$, and the complexity of test stage is equal to  $\mathcal{O}_{\mathrm{ZSL}}$, i.e. $\mathcal{O}(N_uN_aL^2)$.

\section{Experiments}
\label{sec_exper}
To evaluate the performance of the proposed iterative attribute selection model, extensive experiments are conducted on four standard datasets with ZSL setting. In this section, we first compare the proposed approach with the state-of-the-art, and then give detailed analyses.

\subsection{Experimental Settings}
\subsubsection{Dataset} We conduct experiments on four standard ZSL datasets: (1) Animal with Attribute (AwA)~\citep{lampert2014attribute}, (2) attribute-Pascal-Yahoo (aPY)~\citep{farhadi2009describing}, (3) Caltech-UCSD Bird 200-2011 (CUB)~\citep{welinder2010caltech}, and (4) SUN Attribute Database (SUN) \citep{patterson2012sun}. The overall statistic information of these datasets is summarized in Table \ref{table_datasetstat}.

\begin{table*}[t]
	\caption{Statistic Information of Four Datasets (AwA, aPY, CUB and SUN) with Two Dataset Splits (SS and PS).}
	\centering
	\renewcommand\arraystretch{1}
	\setlength{\tabcolsep}{1mm}{
		\begin{tabular}{|c||c|ccc|cc|cc|}
			\hline
			\multirow{2}{*}{\textbf{Dataset}} & \multirow{2}{*}{\textbf{\#Attributes}} & \multicolumn{3}{c|}{\textbf{Classes}} & \multicolumn{2}{c|}{\textbf{Images (SS)}} & \multicolumn{2}{c|}{\textbf{Images (PS)}} \\
			{} & {} & \textbf{\#Total} &\textbf{\#Training} & \textbf{\#Test} & \textbf{\#Training} & \textbf{\#Test} &\textbf{\#Training} & \textbf{\#Test} \\ \hline\hline			
			{AwA } & {85} & {50} & {40} & {10} & {24295} & {6180} & {19832}  & {5685} \\ 
			{aPY } & {64} & {32} & {20} & {12}  & {12695} & {2644}  & {5932}  & {7924} \\ 
			{CUB } & {312} & {200} & {150} & {50}  & {8855} & {2933} & {7057}  & {2967} \\
			{SUN } & {102} & {717} & {645} & {72}  & {12900} & {1440}  & {10320}  & {1440} \\
			\hline			
		\end{tabular}
	}
	\label{table_datasetstat}	
\end{table*}

\subsubsection{Dataset Split} Zero-shot learning assumes that training classes and test classes are disjoint. Actually, ImageNet, the dataset exploited to extract image features via deep neural networks, may include some test classes. Therefore, \citet{xian2018zero} proposed a new dataset split (PS) ensuring that none of the test classes appears in the dataset used to train the extractor model. In this paper, we evaluate the proposed model using both splits, i.e., the original standard split (SS) and the proposed split (PS).

\subsubsection{Image Feature} Deep neural network feature is extracted for the experiments. Image features are extracted from the entire images for AwA, CUB and SUN datasets, and from bounding boxes mentioned in~\citet{farhadi2009describing} for aPY dataset, respectively. The original ResNet-101~\citep{he2016deep} pre-trained on ImageNet with 1K classes is used to calculate 2048-dimensional top-layer pooling units as image features.

\subsubsection{Attribute Representation} Attributes are used as the semantic representation to transfer information from training classes to test classes. We use 85, 64, 312 and 102-dimensional continuous value attributes for AwA, aPY, CUB and SUN datasets, respectively.

\subsubsection{Evaluation protocol} 
Unified dataset splits shown in Table \ref{table_datasetstat} are used for all the compared methods to get fair comparison results. Since the dataset is not well balanced with respect to the number of images per class \citep{xian2018zero}, we use the mean class accuracy, i.e. per-class averaged top-1 accuracy, as the criterion of assessment. Mean class accuracy is calculated as follows: 
\begin{equation}
acc=\frac{1}{L} \sum_{y\in\mathcal{Y}_u}\frac{\# \mathrm{correct\; predictions\; in\;} y}{ \# \mathrm{samples\; in\;} y},
\end{equation}
where $L$ is the number of test classes, $\mathcal{Y}_u$ is the set comprised of all the test labels.

\subsection{Comparison with the State-of-the-Art}
To evaluate the efficiency of the proposed iterative attribute selection model, we modify several latest ZSL baselines by the proposed IAS and compare them with the state-of-the-art.

We modify seven representative ZSL baselines to evaluate the IAS model, including three popular ZSL baselines (i.e. DAP \citep{lampert2014attribute}, LatEm \citep{xian2016latent} and SAE \citep{kodirov2017semantic}) and four latest ZSL baselines (i.e. MFMR \citep{xu2017matrix}, GANZrl \citep{tong2018adversarial}, fVG \citep{xian2019f} and LLAE \citep{li2019zero}).

The improvement achieved on these ZSL baselines is summarized in Table \ref{table_comparison}. It can be observed that IAS can significantly improve the performance of attribute-based ZSL methods. Specifically, the mean accuracies of these ZSL methods on four datasets (i.e. AwA, aPY, CUB and SUN) are increased by $11.09\%$, $15.97\%$, $9.10\%$, $5.11\%$, respectively ($10.29\%$ on average) after using IAS. For DAP on AwA and aPY datasets, LatEm on AwA dataset, IAS can improve their accuracy by greater than $20\%$, which demonstrates that IAS can significantly improve the performance of ZSL models. Interestingly, SAE performs badly on aPY and CUB datasets, while the accuracy rises to an acceptable level (from $8.33\%$ to $38.53\%$, and from $24.65\%$ to $42.85\%$, respectively) by using IAS. Even though the performance of state-of-the-art baselines is pretty well, IAS can still improve them to some extent ($5.48\%$, $3.24\%$, $2.80\%$ and $3.64\%$ on average for MFMR, GANZrl, fVG and LLAE respectively). These results demonstrate that the proposed iterative attribute selection model makes sense and can effectively improve existing attribute-based ZSL methods. This also proves the necessity and effectiveness of attribute selection for ZSL tasks.

\begin{table*}[t]
	\small
	\centering
	\renewcommand\arraystretch{1}
	\caption{Zero-Shot Classification Accuracy Comparison on Benchmarks. Numbers in Brackets are Relative Performance Gains. `-' Indicates that no Reported Results are Available.}
	\label{table_comparison}
	\scalebox{0.72}{
		\setlength{\tabcolsep}{0.5mm}{	
			\begin{tabular}{|c||cc|cc|cc|cc|}
				\hline
				\multirow{2}{*}{\textbf{Methods}} &\multicolumn{2}{c|}{\textbf{AwA}} & \multicolumn{2}{c|}{\textbf{aPY}} & \multicolumn{2}{c|}{\textbf{CUB}} & \multicolumn{2}{c|}{\textbf{SUN}} \\
				& \textbf{SS}          & \textbf{PS}         & \textbf{SS}          & \textbf{PS}         & \textbf{SS}          & \textbf{PS}         & \textbf{SS}         & \textbf{PS}         \\ \hline\hline
				DAP${\ddagger}$ & 64.44          & 46.22         & 35.73          & 39.67         & 43.47          & 40.23         & 41.25         & 45.83         \\
				DAP+AS${\dagger}$ & -          & 48.29         & -          & 34.87         & -            & 41.55          & -          & 42.27          \\
				DAP+IAS & 86.65(+\textit{22.21})           &  {71.88}(+\textit{25.66})         &  57.12(+\textit{21.39})         & 43.06(+\textit{3.39})          &   55.35(+\textit{11.88})         & {54.22}(+\textit{13.99})          & 47.85(+\textit{6.60})          &  50.56(+\textit{4.73})         \\ \hline
				LatEm${\ddagger}$& 71.51           & 48.33          & 24.43           &  34.66         & 50.38           & 48.57          & 58.75          & 55.13          \\
				LatEm+AS${\dagger}$ & -           & 59.07          & -           & 38.82         &   -         &  52.82         &  -         & 58.09          \\
				LatEm+IAS & 81.83(+\textit{10.32})           & 67.13(+\textit{18.80})          & 47.22(+\textit{22.79})           & {48.36}(+\textit{13.70})          & {56.05}(+\textit{5.67})           &  52.14(+\textit{3.57})         & {59.03}(+\textit{0.28})          & 56.18(+\textit{1.05})          \\ \hline
				SAE${\ddagger}$ &  79.19          &  48.48         &  8.33          & 8.33          &  26.41          & 24.65          & 36.94          &   32.78        \\
				SAE+IAS &  {87.95}(+\textit{8.76})          & 70.36(+\textit{21.88})          &  45.90(+\textit{37.57})          & 38.53(+\textit{30.20})          & 48.21(+\textit{21.80})           & 42.85(+\textit{18.20})          & 45.14(+\textit{8.20})          &  42.22(+\textit{9.44})         \\ \hline
				MFMR${\ddagger}$ &86.06            & 68.04          & 52.16           &34.09           & 43.09           & 39.55          &  50.49         & 53.33          \\
				MFMR+IAS & 87.10(+\textit{1.04}) & 71.37(+\textit{3.33}) &{58.51}(+\textit{6.35}) &37.67(+\textit{3.58}) &51.40(+\textit{8.31}) &47.89(+\textit{8.34}) &58.47(+\textit{7.98}) &{58.26}(+\textit{4.93}) \\ \hline
				GANZrl${\dagger}$ &86.23 &- &- &- &62.56 &- &- &-  \\
				GANZrl+IAS &88.51(+\textit{2.28}) &- &- &- &66.76(+\textit{4.20}) &- &- &-   \\ \hline				
				fVG${\dagger}$ &- &70.30 &- &- &- &72.90 &- &65.60          \\
				fVG+IAS &- &74.28(+\textit{3.98}) &- &- &- &74.53(+\textit{1.63}) &- &68.39(+\textit{2.79}) \\ \hline							
				LLAE${\dagger}$ &85.24 &- &56.16 &- &61.93  &- &-  &-     \\
				LLAE+IAS &88.95(+\textit{3.71})  &- &60.88(+\textit{4.72}) &- &64.42(+\textit{2.49}) &-  &-  &-  \\ \hline
				\multicolumn{9}{l}{${\dagger}$:Results published in the paper. \quad ${\ddagger}$:Results reproduced.}
			\end{tabular}
	}}
\end{table*}

As a similar work to ours, ZSLAS selects attributes based on the distributive entropy and the predictability of attributes. Thus, we compare the improvement of IAS and ZSLAS on DAP and LatEm, respectively. In Table \ref{table_comparison}, it can be observed that ZSLAS can improve existing ZSL methods, while IAS can improve them by a greater level ($2.15\%$ vs $10.61\%$ on average). Compared to ZSLAS, the advantages of ZSLIAS can be interpreted in two aspects. Firstly, ZSLIAS selects attributes in an iterative manner, hence it can select a more optimal subset of key attributes than ZSLAS that selects attributes at once. Secondly, ZSLAS is conducted based on the training data, while ZSLIAS is conducted based on the out-of-the-box data which has a similar distribution to the test data. Therefore, attributes selected by ZSLIAS is more applicable and discriminative for test data. Experimental results demonstrate the significant superiority of the proposed IAS model over previous attribute selection models.

\subsection{Detailed Analysis}
In order to further understand the promising performance, we analyze the following experimental results in detail.

\subsubsection{Evaluation on the Out-of-the-box Data} 
In the first experiment, we evaluate the out-of-the-box data generated by a tailor-made attribute-based deep generative model. Figure \ref{fig_tsne} shows the distribution of the out-of-the-box data and the real test data sampled from AwA dataset using t-SNE. Note that the out-of-the-box data in Figure \ref{fig_tsne}(b) is generated only based on the attribute representation of unseen classes, and without extra information of any test images. It can be observed that the generated out-of-the-box data can capture a similar distribution to the real test data, which guarantees that the selected attributes can be effectively generalized to test data.

\begin{figure}[t] 
	\hfill
	\begin{center}
		\includegraphics[width = 0.8\textwidth]{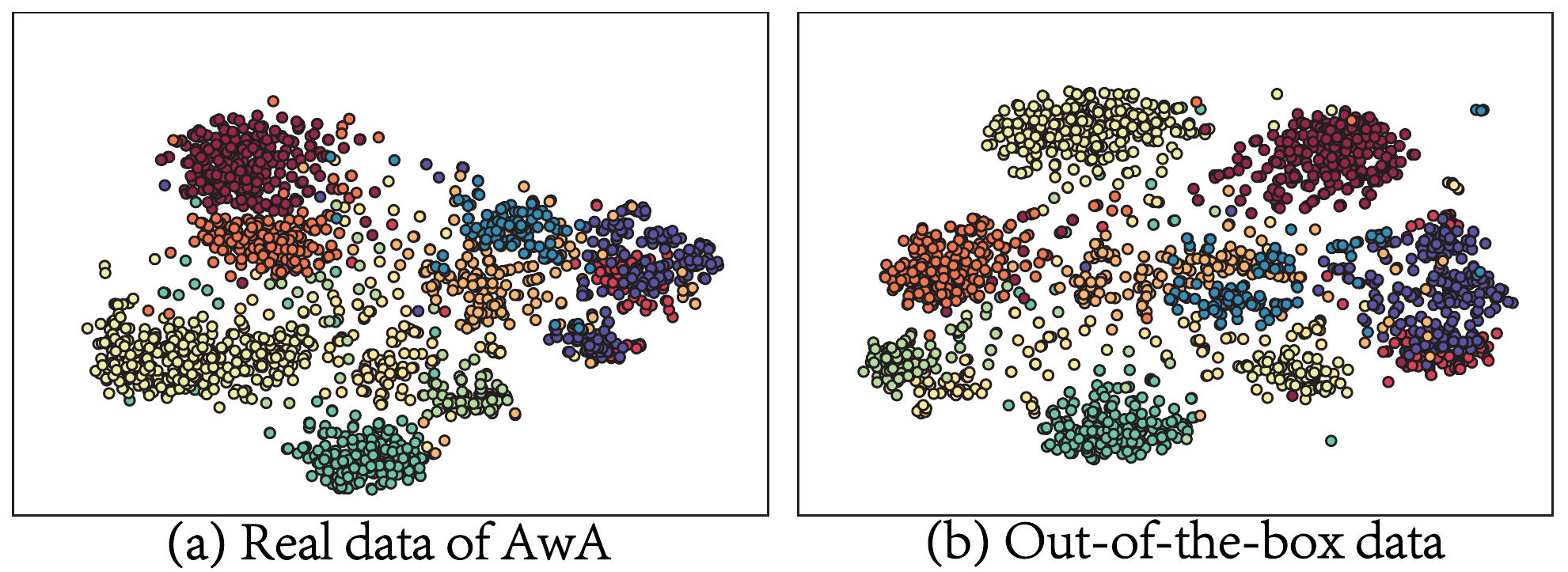}  
	\end{center}
	\caption{T-SNE visualization of the generated out-of-the-box data and real test data of AwA.}
	\label{fig_tsne}
\end{figure}

We also quantitatively evaluate the out-of-the-box data by calculating various distances between three distributions, i.e. the generated out-of-the-box data ($ \mathcal{X}_g $), unseen test data ($ \mathcal{X}_u $) and seen training data ($ \mathcal{X}_s $), in pairs. Table \ref{table_distributionDistance} shows the distribution distances measured by Wasserstein Distance \citep{vallender1974calculation}, KL Divergence \citep{kullback1987letter}, Hellinger Distance \citep{beran1977minimum} and Bhattacharyya Distance \citep{kailath1967divergence}, respectively. It is obvious that the distance between $ \mathcal{X}_g $ and $ \mathcal{X}_u $ is much less than the distance between $ \mathcal{X}_u $ and $ \mathcal{X}_s $, which means that the generated out-of-the-box data has a similar distribution to the unseen test data compared to the seen data. Therefore, attributes selected based on the out-of-the-box data are more discriminative for test data comparing to attributes selected based on training data.

\begin{table}[t]
	\small
	\centering
	\renewcommand\arraystretch{1}
	\caption{Distances Between Different Data Distributions. $ \mathcal{X}_g $ Indicates the Generated Out-of-the-Box Data, $ \mathcal{X}_u $ Indicates the Unseen Test Data and $ \mathcal{X}_s $ Indicates the Seen Training Data.}
	\label{table_distributionDistance}
	\setlength{\tabcolsep}{2mm}{
		\begin{tabular}{|c||c|c|c|}
			\hline
			\textbf{Metrics} &$ \mathcal{X}_g \sim \mathcal{X}_u$ &$ \mathcal{X}_g \sim \mathcal{X}_s$ &$ \mathcal{X}_s \sim \mathcal{X}_u$ \\ \hline\hline
			Wasserstein Distance &\textbf{5.99} &19.09 &18.97 \\
			KL Divergence &\textbf{0.321} &0.630 &0.703 \\
			Hellinger Distance  &\textbf{7.78} &16.87 &17.15 \\
			Bhattacharyya Distance  &\textbf{0.0808} &0.159 &0.176 \\
			\hline
	\end{tabular}}
\end{table}

We illustrate some generated images of unseen classes (i.e. \textit{panda} and \textit{seal}) and annotate them the corresponding attribute representations as shown in Figure \ref{figure_genimg}. Numbers in black indicate the attribute representations of the labels of real test images. Numbers in red and green are the correct and the incorrect attribute values of generated images, respectively. We can see that the generated images have the similar attribute representation as test images. Therefore, the tailor-made attribute-based deep generative model can generate the out-of-the-box data which captures a similar distribution to the unseen data.

\begin{figure*}[t] 
	\hfill
	\begin{center}
		\includegraphics[width = 0.95\textwidth]{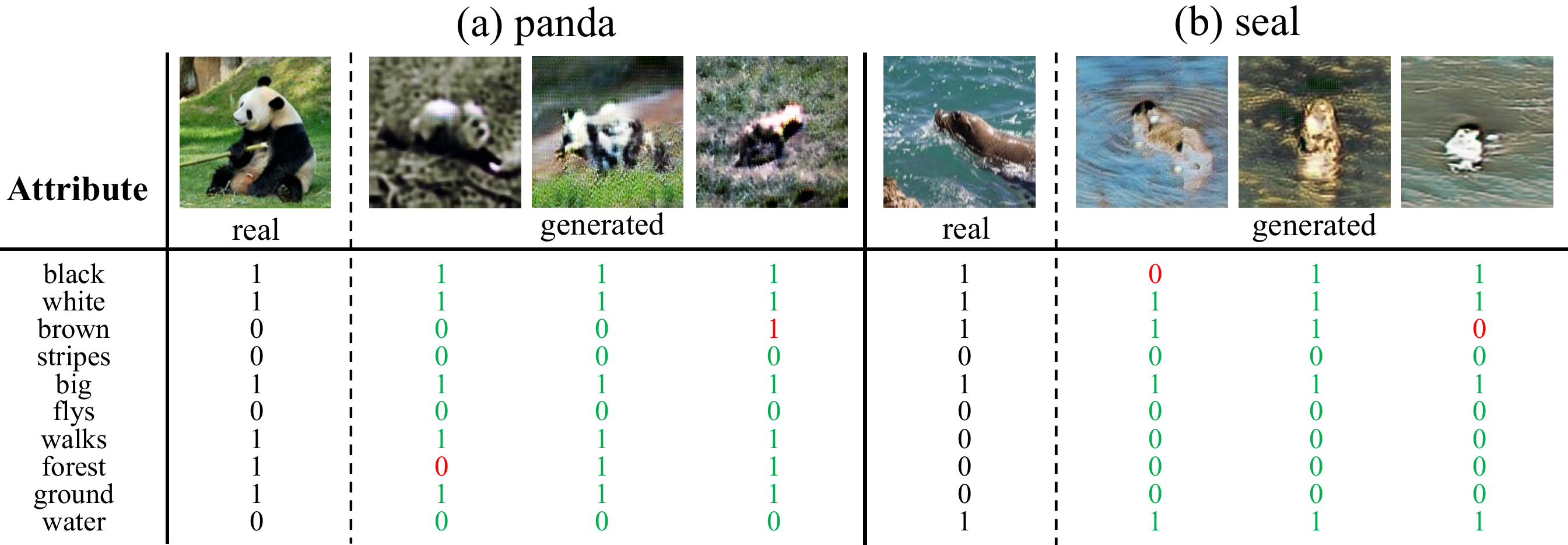}  
	\end{center}
	\caption{Visualization of generated out-of-the-box images and their attribute representation. The first column of both (a) and (b) is the real image derived from AwA dataset. The remaining three columns of both parts are randomly selected from the generated data. Numbers in black are the ground-truth attributes of the real image. Numbers in green and red are the correct and the incorrect attribute values of the generated images, respectively.}
	\label{figure_genimg}
\end{figure*}

\subsubsection{Effectiveness of IAS} 
In the second experiment, we compare the performance of three ZSL methods (i.e. DAP, LatEm and SAE) after using IAS on four datasets, respectively. The accuracies with respect to the number of selected attributes are shown in Figure \ref{fig_24}. On AwA, aPY and SUN datasets, we can see that the performance of these three ZSL methods increases sharply when the number of selected attributes grows from $0$ to about $20\%$, and then reaches the peak. These results suggest that only about a quarter of attributes are the key attributes which are necessary and effective to classify test objects. In Figure \ref{fig_24}(b) and \ref{fig_24}(f), there is an interesting result that SAE performs badly on aPY dataset with both SS and PS (the accuracy is less than $10\%$), while the performance is acceptable after using IAS (the accuracy is about $40\%$). These results demonstrate the effectiveness and robustness of IAS for ZSL tasks.

\begin{figure*}[t] 
	\hfill
	\begin{center}
		\includegraphics[width = 0.95\textwidth]{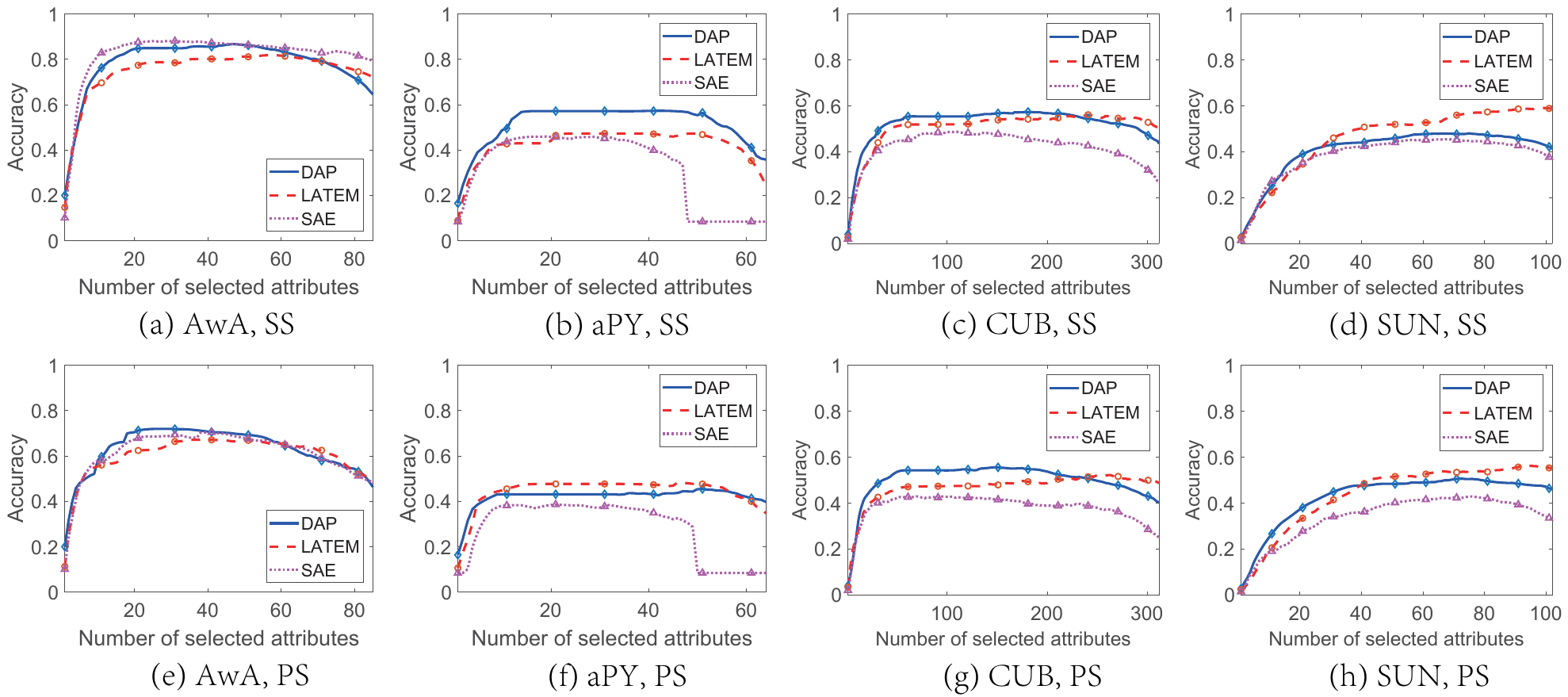}  
	\end{center}
	\caption{Performance of IAS for DAP, LatEm and SAE. The performance of baselines without IAS is shown on the rightmost side of the curves.}
	\label{fig_24}
\end{figure*}

Furthermore, we modify DAP by using all the attributes ($\#84$), using the selected attributes ($\#20$) and using the remaining attributes ($\#64$) after attribute selection, respectively. The resulting confusion matrices of these three variants evaluated on AwA dataset with proposed split setting are illustrated in Figure \ref{fig_confusion}. The numbers in the diagonal area (yellow patches) of confusion matrices indicate the classification accuracy per class. It is obvious that IAS can significantly improve DAP performance on most of the test classes, and the accuracies on some classes nearly doubled after using IAS, such as \textit{horse}, \textit{seal}, and \textit{giraffe}. Even though some objects are hard to be recognized by DAP, like \textit{dolphin} (the accuracy of DAP is $1.6\%$), we can get an acceptable performance after using IAS (the accuracy of DAPIAS is $72.7\%$). The original DAP only performs better than IAS with regard to the object \textit{blue whale}, this is because in the original DAP, most of the marine creatures (such as \textit{blue whale}, \textit{walrus} and \textit{dolphin}) are classified as the blue whale, which increases the classification accuracy while also increasing the false positive rate. More importantly, the confusion matrix of DAPIAS contains less noise (i.e. smaller numbers in the side regions (white patches) of confusion matrices apart from the diagonal area) than DAP, which suggests that DAPIAS has less prediction uncertainties. In other words, adopting IAS can improve the robustness of attribute-based ZSL methods.

\begin{figure*}[t] 
	\hfill
	\begin{center}
		\includegraphics[width = 0.95\textwidth]{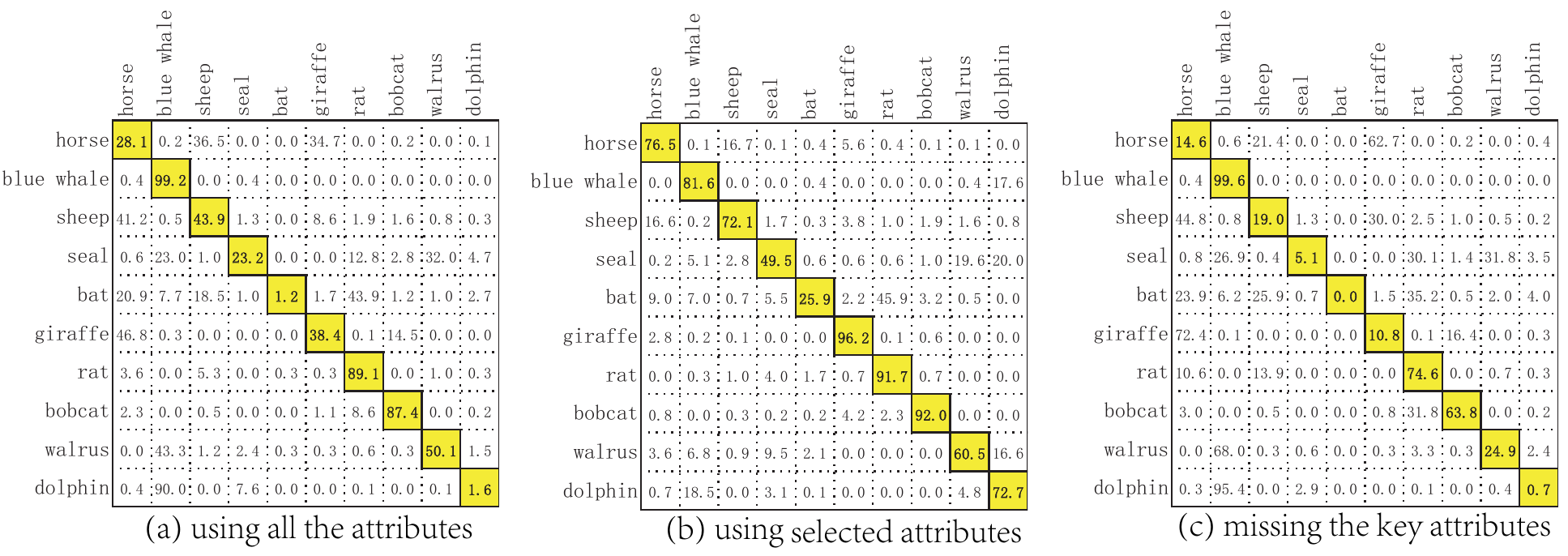}  
	\end{center}
	\caption{Confusion matrices (in \%) between 10 test classes on AwA dataset with proposed split. (a) DAP using all the original attributes; (b) DAP using the key attributes selected by IAS; (c) DAP using the remaining attributes after selection. }
	\label{fig_confusion}
\end{figure*}

In Figure \ref{fig_confusion}, the accuracy of using the selected attributes ($71.88\%$ on average) is significantly improved comparing to the accuracy of using all the attributes ($46.23\%$ on average), and the accuracy of using the remaining attributes ($31.32\%$ on average) is extremely terrible. These results suggest that the selected attributes are the key attributes for discriminating test data. The missing attributes are useless and even have a negative impact on the ZSL system. Therefore, it is obvious that not all the attributes are effective for ZSL tasks, and we should select the key attributes to improve performance.

\begin{table*}[t]
	\small
	\centering
	\renewcommand\arraystretch{1}
	\caption{Subsets of the Key Attributes Selected by DAP, LatEm and SAE on AwA Dataset(20 Attributes Selected Out of 85 Attributes). Attributes Appear in All Three Methods are in Boldface, and Appear in Two Methods are in Italics.}
	\label{table_att}
	\setlength{\tabcolsep}{1mm}{
		\begin{tabular}{|cc|cc|cc|}
			\hline
			\multicolumn{2}{|c|}{\textbf{DAP}} &\multicolumn{2}{|c|}{\textbf{LatEm}} &\multicolumn{2}{|c|}{\textbf{SAE}} \\ \hline\hline
			\textbf{ground} &\textbf{fish} &\textbf{hands} &\textbf{pads} &\textbf{black} &\textbf{paws} \\
			\textbf{hands} &\textbf{fields} &\textbf{ground} &\textbf{forest} &\textbf{ground} &ocean \\
			plains &smelly &bipedal &\textbf{gray} &\textbf{pads} &\textbf{yellow} \\
			\textit{tunnels} &\textbf{pads} &claws &coastal &\textbf{gray} &group \\
			\textbf{forest} &\textbf{yellow} &\textbf{black} &\textbf{yellow} &\textbf{hands} &\textit{tunnels} \\
			\textbf{tail} &\textbf{scavenger} &\textbf{fish} &strainteeth &\textbf{hooves} &\textit{white} \\
			\textbf{gray} &{swims} &\textbf{fields} &horns &domestic &\textbf{fish} \\
			hibernate &\textbf{black} &\textbf{paws} &\textbf{scavenger} &\textbf{tail} &\textbf{fields} \\
			\textbf{hooves} &\textbf{paws} &blue &\textbf{tail} &skimmer &\textbf{forest} \\
			jungle &weak &\textbf{hooves} &\textit{white} &arctic &\textbf{scavenger} \\ \hline
	\end{tabular}}
\end{table*}

\subsubsection{Interpretability of Selected Attributes} 
In the third experiment, we present the visualization results of attribute selection. We find that ZSL methods obtain the best performance when selecting about $20\%$ attributes as shown in Figure \ref{fig_24}. Therefore, we illustrate the top $20\%$ key attributes selected by DAP, LatEm and SAE on four datasets in Figure \ref{fig_subset}. Three rows in each figure are DAP, LatEm and SAE from top to bottom, and yellow bars indicate the attributes which are selected by the corresponding methods. We can see that the attribute subsets selected by different ZSL methods are highly coincident for the same dataset, which demonstrates that the selected attributes are the key attributes for discriminating test data. Specifically, we enumerate the key attributes selected by three ZSL methods on AwA dataset in Table \ref{table_att}. Attributes in boldface indicate that they are simultaneously selected by all the three ZSL methods, and attributes in italics indicate that they are selected by any two of these three methods. It can be observed that 13 attributes ($65\%$) are selected by all the three ZSL methods. These three attribute subsets selected by diverse ZSL models are very similar, which is another evidence that IAS is reasonable and useful for zero-shot classification.

\begin{figure*}[!t] 
	\hfill
	\begin{center}
		\includegraphics[width = 0.95\textwidth]{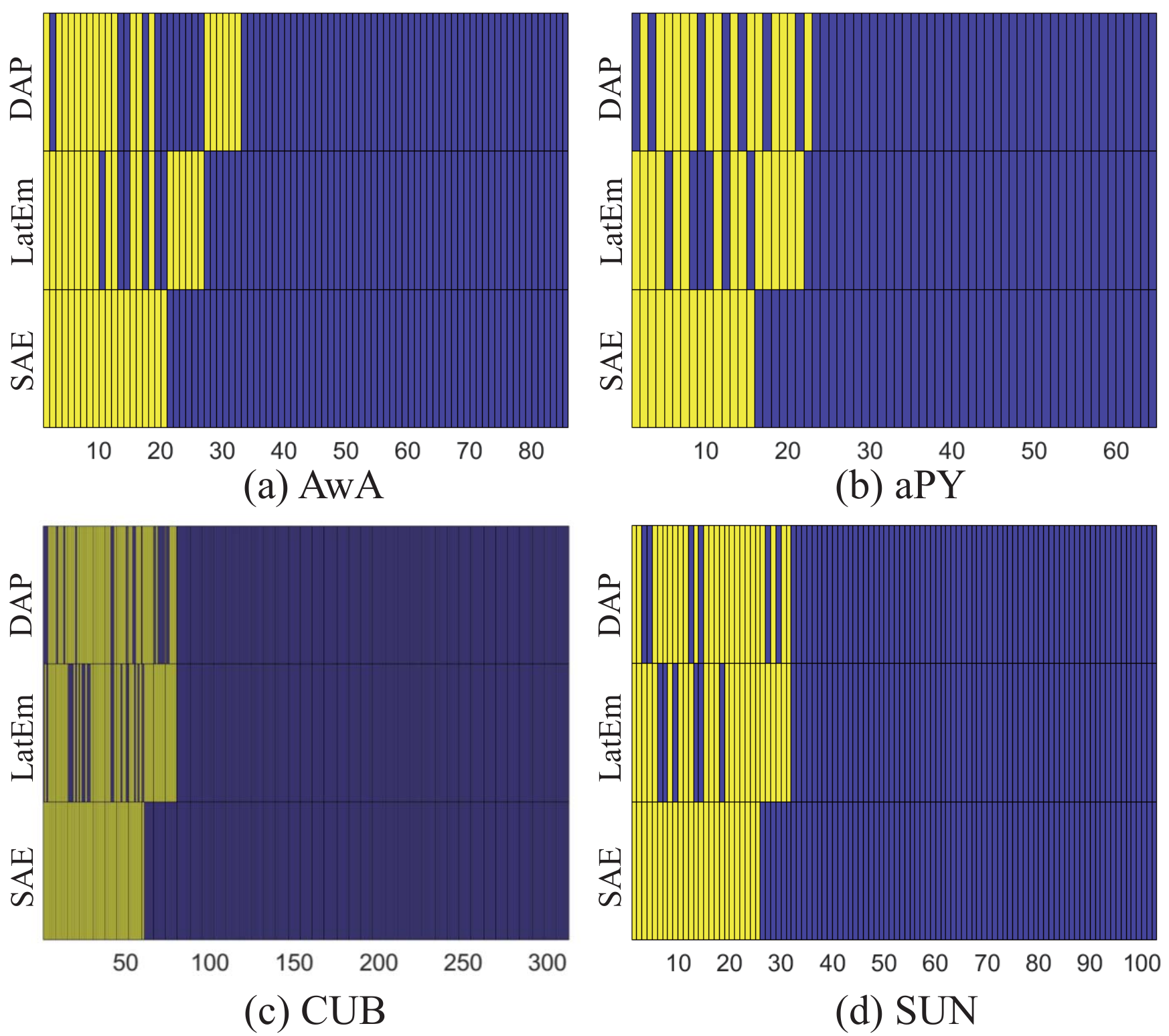}  
	\end{center}
	\caption{Visualization of attribute subsets selected by three different ZSL methods on four datasets. Three rows in each figure are DAP, LatEm and SAE from top to bottom. The horizontal axis represents the attribute, and yellow bars indicate the attributes selected by the corresponding methods.}
	\label{fig_subset}
\end{figure*}

\section{Conclusion}
\label{sec_conclu}
We present a novel and effective iterative attribute selection model to improve existing attribute-based ZSL methods. In most of the previous ZSL works, all the attributes are assumed to be effective and treated equally. However, we notice that attributes have different predictability and discriminability for diverse objects. Motivated by this observation, we propose to select the key attributes to build ZSL model. Since training classes and test classes are disjoint in ZSL tasks, we introduce the out-of-the-box data to mimic test data to guide the progress of attribute selection. The out-of-the-box data generated by a tailor-made attribute-based deep generative model has a similar distribution to the test data. Hence, the attributes selected by IAS based on the out-of-the-box data can be effectively generalized to the test data. To evaluate the effectiveness of IAS, we conduct extensive experiments on four standard ZSL datasets. Experimental results demonstrate that IAS can effectively select the key attributes for ZSL tasks and significantly improve state-of-the-art ZSL methods.

In this work, we select the same attributes for all the unseen test classes. Obviously, this is not the global optimal solution to select attributes for diverse categories. In the future, we will consider a tailor-made attribute selection model that can select the special subset of key attributes for each test class.

\section*{Acknowledgments}
This work is supported in part by ARC under Grant LP150100671 and Grant DP180100106, in part by NSFC under Grant 61373063 and Grant 61872188, in part by the Project of MIIT under Grant E0310/1112/02-1, in part by the Collaborative Innovation Center of IoT Technology and Intelligent Systems of Minjiang University under Grant IIC1701, and in part by China Scholarship Council.

\bibliographystyle{APA}

\begin{thebibliography}{100}
\providecommand{\natexlab}[1]{#1}
\expandafter\ifx\csname urlstyle\endcsname\relax
  \providecommand{\doi}[1]{doi:\discretionary{}{}{}#1}\else
  \providecommand{\doi}{doi:\discretionary{}{}{}\begingroup
  \urlstyle{rm}\Url}\fi


\bibitem[{Airola et~al.(2017)Airola, \& Pahikkala}]{airola2018fast}
Airola, A., \& Pahikkala, T. (2017).
\newblock Fast Kronecker product kernel methods via generalized vec trick.
\newblock \emph{IEEE Transactions on Neural Networks and Learning Systems}, \emph{29(8)}, 3374 -- 3387.

\bibitem[{Akata et~al.(2015)Akata, Perronnin, Harchaoui, \& Schmid}]{akata2016label}
Akata, Z., Perronnin, F., Harchaoui, Z., \& Schmid, C. (2015).
\newblock Label-embedding for image classification.
\newblock \emph{IEEE Transactions on Pattern Analysis and Machine Intelligence}, \emph{38(7)}, 1425 -- 1438.

\bibitem[{Beran(1977)}]{beran1977minimum}
Beran, R. (1977).
\newblock Minimum Hellinger distance estimates for parametric models.
\newblock \emph{The Annals of Statistics}, \emph{5(3)}, 445 -- 463.

\bibitem[{Bucher et~al.(2017)Bucher, Herbin, \& Jurie}]{bucher2017generating}
Bucher, M., Herbin, S., \& Jurie, F. (2017).
\newblock Generating visual representations for zero-shot classification.
\newblock In \emph{Proceedings of the IEEE International Conference on Computer Vision} (pp. 2666-2673).

\bibitem[{Burges et~al.(2005)Burges, Shaked, Renshaw, Lazier, Deeds, Hamilton, \& Hullender}]{burges2005learning}
Burges, C., Shaked, T., Renshaw, E., Lazier, A., Deeds, M., Hamilton, N., \& Hullender, G. (2005).
\newblock Learning to rank using gradient descent.
\newblock In \emph{Proceedings of the 22nd International Conference on Machine learning (ICML-05)} (pp. 89 -- 96).

\bibitem[{Cheng et~al.(2017)Cheng, Qiao, Wang, \& Yu}]{cheng2018random}
Cheng, Y., Qiao, X., Wang, X., \& Yu, Q. (2017).
\newblock Random forest classifier for zero-shot learning based on relative attribute.
\newblock \emph{IEEE Transactions on Neural Networks and Learning Systems}, \emph{29(5)}, 1662 -- 1674.

\bibitem[{Cormen et~al.(2009)Cormen, Leiserson, Rivest, \& Stein}]{cormen2009introduction}
Cormen, T. H., Leiserson, C. E., Rivest, R. L., \& Stein, C. (2009).
\newblock \emph{Introduction to Algorithms}.
\newblock MIT press.

\bibitem[{Dietterich et~al.(1994)Dietterich, \& Bakiri}]{dietterich1994solving}
Dietterich, T. G., \& Bakiri, G. (1994).
\newblock Solving multiclass learning problems via error-correcting output codes.
\newblock \emph{Journal of Artificial Intelligence Research}, \emph{2)}, 263 -- 286.

\bibitem[{Farhadi et~al.(2009)Farhadi, Endres, Hoiem, \& Forsyth}]{farhadi2009describing}
Farhadi, A., Endres, I., Hoiem, D., \& Forsyth, D. (2009, June).
\newblock Describing objects by their attributes.
\newblock In \emph{2009 IEEE Conference on Computer Vision and Pattern Recognition} (pp. 1778 -- 1785). IEEE.

\bibitem[{Garey et~al.(1974)Garey, Johnson, \& Stockmeyer}]{garey1974some}
Garey, M. R., Johnson, D. S., \& Stockmeyer, L. (1974, April). 
\newblock Some simplified NP-complete problems.
\newblock In \emph{Proceedings of the Sixth Annual ACM Symposium on Theory of Computing} (pp. 47 -- 63). ACM.

\bibitem[{Guo et~al.(2018)Guo, Ding, Han, \& Tang}]{guo2018zero}
Guo, Y., Ding, G., Han, J., \& Tang, S. (2018, April).
\newblock Zero-shot learning with attribute selection.
\newblock In \emph{Thirty-Second AAAI Conference on Artificial Intelligence}.

\bibitem[{He et~al.(2016)He, Zhang, Ren, \& Sun}]{he2016deep}
He, K., Zhang, X., Ren, S., \& Sun, J. (2016).
\newblock Deep residual learning for image recognition.
\newblock In \emph{Proceedings of the IEEE Conference on Computer Vision and Pattern Recognition} (pp. 770 -- 778).

\bibitem[{Ji et~al.(2019)Ji, Sun, Yu, Pang, \& Han}]{ji2019attribute}
Ji, Z., Sun, Y., Yu, Y., Pang, Y., \& Han, J. (2019).
\newblock Attribute-guided network for cross-modal zero-shot hashing.
\newblock \emph{IEEE Transactions on Neural Networks and Learning Systems}.

\bibitem[{Jiang et~al.(2017)Jiang, Wang, Shan, Yang, \& Chen}]{jiang2017learning}
Jiang, H., Wang, R., Shan, S., Yang, Y., \& Chen, X. (2017).
\newblock Learning Discriminative Latent Attributes for Zero-Shot Classification.
\newblock In \emph{Proceedings of the IEEE International Conference on Computer Vision} (pp. 4223 -- 4232).

\bibitem[{Kailath(1967)}]{kailath1967divergence}
Kailath, T. (1967).
\newblock The divergence and Bhattacharyya distance measures in signal selection.
\newblock \emph{IEEE Transactions on Communication Technology}, \emph{15(1)}, 52 -- 60.

\bibitem[{Kingma et~al.(2013)Kingma, \& Welling}]{kingma2013auto}
Kingma, D. P., \& Welling, M. (2013).
\newblock Auto-encoding variational bayes.
\newblock \emph{arXiv preprint arXiv:1312.6114}.

\bibitem[{Kodirov et~al.(2017)Kodirov, Xiang, \& Gong}]{kodirov2017semantic}
Kodirov, E., Xiang, T., \& Gong, S. (2017).
\newblock Semantic autoencoder for zero-shot learning.
\newblock In \emph{Proceedings of the IEEE Conference on Computer Vision and Pattern Recognition} (pp. 3174-3183).

\bibitem[{Kullback(1987)}]{kullback1987letter}
Kullback, S. (1987).
\newblock Letter to the editor: The Kullback-Leibler distance.

\bibitem[{Lampert et~al.(2013)Lampert, Nickisch, \& Harmeling}]{lampert2014attribute}
Lampert, C. H., Nickisch, H., \& Harmeling, S. (2013).
\newblock Attribute-based classification for zero-shot visual object categorization.
\newblock \emph{IEEE Transactions on Pattern Analysis and Machine Intelligence}, \emph{36(3)}, 453 -- 465.

\bibitem[{Li et~al.(2019)Li, Jing, Lu, Zhu, Yang, \& Huang}]{li2019zero}
Li, J., Jing, M., Lu, K., Zhu, L., Yang, Y., \& Huang, Z. (2019).
\newblock From Zero-Shot Learning to Cold-Start Recommendation.
\newblock In \emph{Thirty-Third AAAI Conference on Artificial Intelligence}.

\bibitem[{Liu et~al.(2014)Liu, Wiliem, Chen, \& Lovell}]{liu2014automatic}
Liu, L., Wiliem, A., Chen, S., \& Lovell, B. C. (2014, August).
\newblock Automatic image attribute selection for zero-shot learning of object categories.
\newblock In \emph{2014 22nd International Conference on Pattern Recognition} (pp. 2619 -- 2624). IEEE.

\bibitem[{Liu et~al.(2017)Liu, Dai, Humayun, Tay, Yu, Smith, Rehg, \& Song}]{liu2017iterative}
Liu, W., Dai, B., Humayun, A., Tay, C., Yu, C., Smith, L. B., Rehg, J. M., \& Song, L. (2017, August).
\newblock Iterative machine teaching.
\newblock In \emph{Proceedings of the 34th International Conference on Machine Learning-Volume 70} (pp. 2149 -- 2158). JMLR. org.

\bibitem[{Ma et~al.(2017)Ma, Jia, Sun, Schiele, Tuytelaars, \& Van Gool}]{ma2017pose}
Ma, L., Jia, X., Sun, Q., Schiele, B., Tuytelaars, T., \& Van Gool, L. (2017).
\newblock Pose guided person image generation.
\newblock In \emph{Advances in Neural Information Processing Systems} (pp. 406 -- 416).

\bibitem[{Ma et~al.(2017)Ma, Chang, Xu, Sebe, \& Hauptmann}]{ma2017joint}
Ma, Z., Chang, X., Xu, Z., Sebe, N., \& Hauptmann, A. G. (2017).
\newblock Joint attributes and event analysis for multimedia event detection.
\newblock \emph{IEEE Transactions on Neural Networks and Learning Systems}, \emph{29(7)}, 2921 -- 2930.

\bibitem[{Miao et~al.(2016)Miao, Huang, \& Zhao}]{miao2016topprf}
Miao, J., Huang, J. X., \& Zhao, J. (2016).
\newblock TopPRF: A probabilistic framework for integrating topic space into pseudo relevance feedback.
\newblock \emph{ACM Transactions on Information Systems (TOIS)}, \emph{34(4)}, 22.

\bibitem[{Murphy(2004)}]{murphy2004big}
Murphy, G. (2004).
\newblock \emph{The big book of concepts}.
\newblock MIT press.

\bibitem[{Odena et~al.(2017)Odena, Olah, \& Shlens}]{odena2016conditional}
Odena, A., Olah, C., \& Shlens, J. (2017, August).
\newblock Conditional image synthesis with auxiliary classifier gans.
\newblock In \emph{Proceedings of the 34th International Conference on Machine Learning-Volume 70} (pp. 2642 -- 2651). JMLR. org.

\bibitem[{Palatucci et~al.(2009)Palatucci, M., Pomerleau, Hinton, \& Mitchell}]{palatucci2009zero}
Palatucci, M., Pomerleau, D., Hinton, G. E., \& Mitchell, T. M. (2009).
\newblock Zero-shot learning with semantic output codes.
\newblock In \emph{Advances in Neural Information Processing Systems} (pp. 1410 -- 1418).

\bibitem[{Patterson et~al.(2012)Patterson, \& Hays}]{patterson2012sun}
Patterson, G., \& Hays, J. (2012, June).
\newblock Sun attribute database: Discovering, annotating, and recognizing scene attributes.
\newblock In \emph{2012 IEEE Conference on Computer Vision and Pattern Recognition} (pp. 2751 -- 2758). IEEE.

\bibitem[{Rocha et~al.(2014)Rocha, \& Goldenstein}]{rocha2014multiclass}
Rocha, A., \& Goldenstein, S. K. (2014).
\newblock Multiclass from binary: Expanding one-versus-all, one-versus-one and ecoc-based approaches.
\newblock \emph{IEEE Transactions on Neural Networks and Learning Systems}, \emph{25(2)}, 289 -- 302.

\bibitem[{Romera-Paredes et~al.(2015)Romera-Paredes, \& Torr}]{romera2015embarrassingly}
Romera-Paredes, B., \& Torr, P. (2015, June).
\newblock An embarrassingly simple approach to zero-shot learning.
\newblock In \emph{International Conference on Machine Learning} (pp. 2152 -- 2161).

\bibitem[{Shen et~al.(2018)Shen, Ji, Wang, Li, \& Li}]{shen2018weakly}
Shen, Y., Ji, R., Wang, C., Li, X., \& Li, X. (2018).
\newblock Weakly supervised object detection via object-specific pixel gradient.
\newblock \emph{IEEE Transactions on Neural Networks and Learning Systems}, \emph{29(12)}, 5960 -- 5970.

\bibitem[{Stock et~al.(2018)Stock, Pahikkala, Airola, De Baets, \& Waegeman}]{stock2018comparative}
Stock, M., Pahikkala, T., Airola, A., De Baets, B., \& Waegeman, W. (2018).
\newblock A comparative study of pairwise learning methods based on kernel ridge regression.
\newblock \emph{Neural Computation}, \emph{30(8)}, 2245 -- 2283.

\bibitem[{Sohn et~al.(2015)Sohn, Lee, \& Yan}]{sohn2015learning}
Sohn, K., Lee, H., \& Yan, X. (2015).
\newblock Learning structured output representation using deep conditional generative models.
\newblock In \emph{Advances in Neural Information Processing Systems} (pp. 3483 -- 3491).

\bibitem[{Sun et~al.(2017)Sun, Schiele, \& Fritz}]{sun2017domain}
Sun, Q., Schiele, B., \& Fritz, M. (2017).
\newblock A domain based approach to social relation recognition.
\newblock In \emph{Proceedings of the IEEE Conference on Computer Vision and Pattern Recognition} (pp. 3481 -- 3490).

\bibitem[{Tong et~al.(2018)Tong, Klinkigt, Chen, Cui, Kong, Murakami, \& Kobayashi}]{tong2018adversarial}
Tong, B., Klinkigt, M., Chen, J., Cui, X., Kong, Q., Murakami, T., \& Kobayashi, Y. (2018, April).
\newblock Adversarial zero-shot learning with semantic augmentation.
\newblock In \emph{Thirty-Second AAAI Conference on Artificial Intelligence.} (pp. 3483 -- 3491).

\bibitem[{Valiant(1984)}]{valiant1984theory}
Valiant, L. G. (1984, December).
\newblock A theory of the learnable.
\newblock In \emph{Proceedings of the Sixteenth Annual ACM Symposium on Theory of Computing} (pp. 436 -- 445). ACM.

\bibitem[{Vallender(1974)}]{vallender1974calculation}
Vallender, S. (1974).
\newblock Calculation of the Wasserstein distance between probability distributions on the line.
\newblock \emph{Theory of Probability \& Its Applications}, \emph{18(4)}, 784 -- 786.

\bibitem[{Vapnik(2013)}]{vapnik2013nature}
Vapnik, V. (2013).
\newblock \emph{The nature of statistical learning theory}.
\newblock Springer Science \& Business Media.

\bibitem[{Wah et~al.(2011)Wah, Branson, Welinder, Perona, \& Belongie}]{welinder2010caltech}
Wah, C., Branson, S., Welinder, P., Perona, P., \& Belongie, S. (2011).
\newblock The caltech-ucsd birds-200-2011 dataset.

\bibitem[{Xian et~al.(2016)Xian, Akata, Sharma, Nguyen, Hein, \& Schiele}]{xian2016latent}
Xian, Y., Akata, Z., Sharma, G., Nguyen, Q., Hein, M., \& Schiele, B. (2016).
\newblock Latent embeddings for zero-shot classification.
\newblock In \emph{Proceedings of the IEEE Conference on Computer Vision and Pattern Recognition} (pp. 69 -- 77).

\bibitem[{Xian et~al.(2018)Xian, Lampert, Schiele, \& Akata}]{xian2018zero}
Xian, Y., Lampert, C. H., Schiele, B., \& Akata, Z. (2018).
\newblock Zero-shot learning-a comprehensive evaluation of the good, the bad and the ugly.
\newblock \emph{IEEE Transactions on Pattern Analysis and Machine Intelligence}.

\bibitem[{Xian et~al.(2019)Xian, Sharma, Schiele, \& Akata}]{xian2019f}
Xian, Y., Sharma, S., Schiele, B., \& Akata, Z. (2019).
\newblock f-VAEGAN-D2: A feature generating framework for any-shot learning.
\newblock In \emph{Proceedings of the IEEE Conference on Computer Vision and Pattern Recognition} (pp. 10275 -- 10284).

\bibitem[{Xu et~al.(2017)Xu, Shen, Yang, Zhang, Shen, \& Song}]{xu2017matrix}
Xu, X., Shen, F., Yang, Y., Zhang, D., Shen, H., \& Song, J. (2017).
\newblock Matrix tri-factorization with manifold regularizations for zero-shot learning.
\newblock In \emph{Proceedings of the IEEE Conference on Computer Vision and Pattern Recognition} (pp. 3798-3807).

\bibitem[{Xu et~al.(2019)Xu, Tsang, \& Liu}]{xu2019complementary}
Xu, X., Tsang, I. W., \& Liu, C. (2019).
\newblock Complementary Attributes: A New Clue to Zero-Shot Learning.
\newblock \emph{IEEE Transactions on Cybernetics.}

\bibitem[{Zheng et~al.(2018)Zheng, Li, Yan, Tang, \& Tan}]{zheng2018sparse}
Zheng, Y., Li, S., Yan, R., Tang, H., \& Tan, K. C. (2018).
\newblock Sparse temporal encoding of visual features for robust object recognition by spiking neurons.
\newblock \emph{IEEE Transactions on Neural Networks and Learning Systems}, \emph{29(12)}, 5823 -- 5833.

\bibitem[{Zhou et~al.(2019)Zhou, Fang, Zhang, Gong, Peng, Cao, \& Goh}]{zhou2019learning}
Zhou, J. T., Fang, M., Zhang, H., Gong, C., Peng, X., Cao, Z., \& Goh, R. S. M. (2019).
\newblock Learning with annotation of various degrees.
\newblock \emph{IEEE Transactions on Neural Networks and Learning Systems.}.

\bibitem[{Zhou et~al.(2019)Zhou, Tsang, Ho, \& Muller}]{zhou2019n}
Zhou, J. T., Tsang, I. W., Ho, S., \& Muller, K. (2019).
\newblock N-ary decomposition for multi-class classification.
\newblock \emph{Machine Learning}, \emph{108(5)}, 809 -- 830.

\end{thebibliography}

\end{document}